\documentclass[sigconf]{acmart} 
\AtBeginDocument{%
  \providecommand\BibTeX{{%
    \normalfont B\kern-0.5em{\scshape i\kern-0.25em b}\kern-0.8em\TeX}}}
    
\renewcommand\footnotetextcopyrightpermission[1]{} 
\setcopyright{none}

\usepackage{bm}
\usepackage{subcaption}
\usepackage[ruled,vlined]{algorithm2e}
\usepackage{algorithmic}
\usepackage{amsmath,amsthm}
\usepackage{mismath}
\usepackage{mathtools}
\usepackage{enumitem}
\usepackage{booktabs}
\usepackage{stfloats}

\DeclareMathOperator*{\argmax}{arg\,max}

\newcommand{\nosection}[1]{\vspace{2pt}\noindent\textbf{#1.}}

\begin{document}

\title{Vertical Federated Linear Contextual Bandits}

\author{Zeyu Cao}
\authornote{Both authors contributed equally to this research.}
\affiliation{
 \institution{Tencent AI Lab}
 \city{Shenzhen}
 \country{China}}
\email{zc317@cam.ac.uk}

\author{Zhipeng Liang}
\authornotemark[1]
\authornote{Work done during internship at Tencent AI Lab.}
\affiliation{
 \institution{HKUST}
 \city{Hong Kong}
 \country{China}}
\email{zliangao@connect.ust.hk}
 
\author{Shu Zhang}
\affiliation{
 \institution{Tencent}
 \city{Shenzhen}
 \country{China}}
\email{bookzhang@tencent.com}
 
\author{Hangyu Li}
\affiliation{
 \institution{Tencent}
 \city{Shenzhen}
 \country{China}}
\email{masonhyli@tencent.com}
 
\author{Ouyang Wen}
\affiliation{
 \institution{Tencent}
 \city{Shenzhen}
 \country{China}}
\email{gdpouyang@tencent.com}
 
\author{Yu Rong}
\affiliation{
 \institution{Tencent AI Lab}
 \city{Shenzhen}
 \country{China}}
\email{royrong@tencent.com}

\author{Peilin Zhao}
\affiliation{
 \institution{Tencent AI Lab}
 \city{Shenzhen}
 \country{China}}
\email{masonzhao@tencent.com}
 
\author{Bingzhe Wu}
\authornote{Corresponding author.}
\affiliation{
 \institution{Tencent AI Lab}
 \city{Shenzhen}
 \country{China}}
\email{wubingzhe94@gmail.com}

\begin{abstract}

In this paper, we investigate a novel problem of building contextual bandits in the vertical federated setting, i.e., contextual information is vertically distributed over different departments. This problem remains largely unexplored in the research community. To this end, we carefully design a customized encryption scheme named orthogonal matrix-based mask mechanism(O3M) for encrypting local contextual information while avoiding expensive conventional cryptographic techniques. We further apply the mechanism to two commonly-used bandit algorithms, LinUCB and LinTS, and instantiate two practical protocols for online recommendation under the vertical federated setting. The proposed protocols can perfectly recover the service quality of centralized bandit algorithms while achieving a satisfactory runtime efficiency, which is theoretically proved and analyzed in this paper. By conducting extensive experiments on both synthetic and real-world datasets, we show the superiority of the proposed method in terms of privacy protection and recommendation performance.
\end{abstract}


\maketitle
\pagestyle{plain}

\begin{figure}[t!]
    \includegraphics[width=0.8\linewidth]{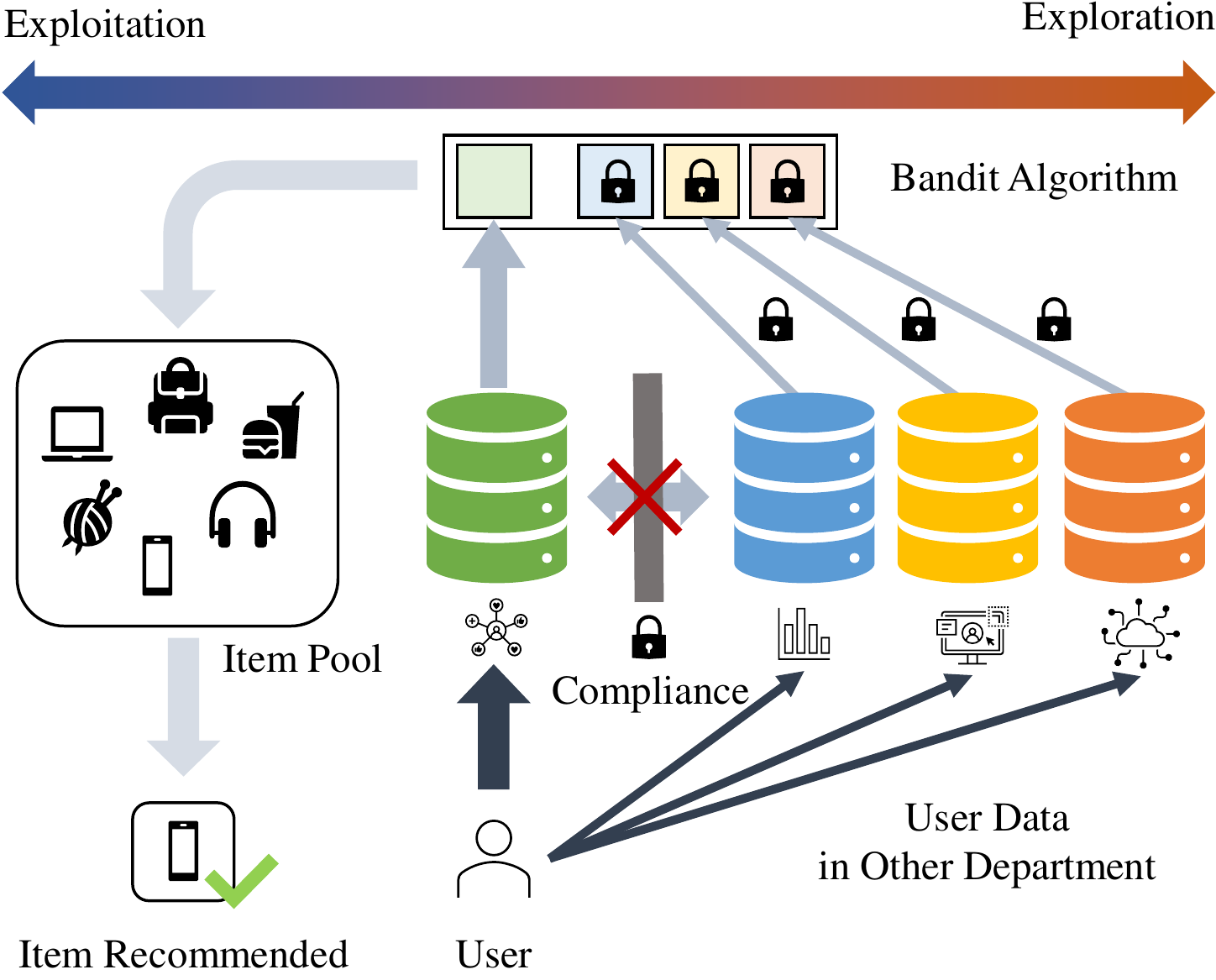}
    \centering
    \caption{Overview on using bandit algorithm for recommendation problem in vertical federated setting. The department A want to make the recommendation for the cold-start user $i$. Since the user is new to the system, instead of only exploiting user data in department A, one good strategy to improve the recommendation results is exploring the user data from other departments. These department's data are illustrated in different colors. However, as there are compliance requirement, the recommendation system cannot directly use the data, and must use privacy preserving techniques during recommendation. This is illustrated with the lock on the data when the contextual bandit is attempted to use the data to conduct recommendation.}
    \label{fig:bandit-description}
\end{figure}

\section{Introduction}
Personalized recommendation system is one of the fundamental building blocks of numerous web service platforms such as E-commerce and online advertising \cite{DBLP:journals/corr/abs-1708-05024}. 
From a technical perspective, most of these recommendation tasks can be modeled as an online sequential decision process wherein the system needs to recommend items to the current user based on sequential dynamics (i.e., historical user-item interactions and user/item contexts). 
In such an online process, the cold-start problem is ubiquitous since new users and items continuously join the online service and these new users typically come with incomplete profiles and they rarely interact with items \cite{DBLP:journals/eswa/LikaKH14}. 
Therefore, cold-start recommendation faces two critical challenges, i.e., lack of user profiles/contexts and lack of user-item interactions. 

The key task of cold-start problems is to efficiently collect user-item interactions when one or both of them are new. This task is challenging since there are two competing goals in such a collection process, i.e.,  exploration and exploitation.
To solve this problem, contextual bandit-based approach is seen as a principled way in which a learning
algorithm sequentially recommends an item to users based on contextual information of the user and item, while simultaneously
adapting its recommendation strategy based on historical user-click feedback to maximize total user clicks in the long run. 
Theoretically, seminal contextual bandit algorithms, including LinUCB \cite{DBLP:conf/nips/Abbasi-YadkoriPS11} and LinTS \cite{agrawal2013thompson} are proved to achieve the optimal balance between exploration and exploitation whereas empirically, these approaches have shown remarkable performance and have been deployed into various industrial scenarios for dealing with cold start recommendation \cite{lattimore2020bandit, slivkins2019introduction, DBLP:conf/www/LiCLS10}.

Despite their remarkable progress, such approaches cannot be directly employed in many real-world scenarios where the recommendation service provider only has some portions of the user's profiles while the other user features belong to different departments and cannot be directly shared with the recommendation provider due to privacy concerns or data regulations. 
For example, in the e-commerce recommendation setting, the service provider (e.g., Amazon and Taobao) can only hold shopping-related information while financial information such as deposits is held by other departments. 
This setting is documented as the vertical federated setting in the literature \cite{DBLP:conf/kdd/ChaiWZYC0022}. 
A question is naturally raised: \textit{How to design the contextual bandits in the vertical federated setting?}

To solve this problem, a natural idea is to apply modern privacy enhancing techniques such as secure multi-party computation (MPC) \cite{goldreich1998secure} and differential privacy (DP) \cite{dwork2008differential, dwork2014algorithmic} to current contextual bandit approaches. 
However, this manner typically leads to either high computation/communication costs or the degradation of recommendation performance: For example, in the multi-party computation, the Multiplication operator would lead to ~380 times slower and the comparison operator would even lead to ~34000 times slower \cite{DBLP:journals/corr/abs-1910-05299}. Thus directly applying these cryptographic techniques to the contextual bandit is intractable since there involve large amounts of multiplication/comparison operations. 
For differential privacy, injecting sufficient noise into the contextual information leads to a significant loss of the statistics utility and worse recommendation performance \cite{DBLP:conf/nips/0007C0L020, NEURIPS2021_df0e09d6}.

All these cryptographic techniques are originally designed for general-purpose computations. 
In this paper, we note that the core computation operators in the contextual bandits come with unique properties, thus providing the opportunity to design customized encryption schemes with better runtime efficiency and statistical utility. 
In general, the exploration mechanism is at the core of previous bandit algorithms, which typically use the empirical sample covariance matrix to construct the quantity needed by the exploration mechanism. 
Based on the property of the sample covariance matrix, we propose a technique named \textbf{o}rthogonal \textbf{m}atrix based \textbf{m}ask \textbf{m}echanism (O3M) for computing this quantity in a private and lossless manner, which is motivated by prior works on federated PCA~\cite{DBLP:conf/kdd/ChaiWZYC0022}. O3M is used by the data provider for masking the local contextual information (i.e., user features) before sending them to the service provider.
Specifically, O3M encodes the raw contextual information by multiplying with an orthogonal matrix (i.e., mask) then the data provider will send the masked data to the service provider. 
Since the service provider has no access to the mask, it cannot infer the original information even though it knows the masked data. 
However, due to the mathematical property of the orthogonal matrix, its effect can be removed when used in the construction of the key statistic in the bandit algorithms and thus we can obtain "lossless" performance in a privacy-preserving manner (see proof details in Section~\ref{sec:algorithm}). We further apply the O3M technique to two commonly-used bandit algorithms, LinUCB and LinTS,
and implement two practical protocols named VFUCB and VFTS for real-world recommendation tasks. For simplicity, in the following discussions, we refer these protocols as VFCB (Vertical Federated Contextual Bandit) without contextual ambiguity.

Besides, to help the reader have a  better understanding of our protocols, we conduct comprehensive complexity analysis and provide formal utility and security proofs in Section~\ref{sec:algorithm}.
By performing extensive empirical studies on both synthetic and real-world datasets, we show that the proposed protocols can achieve similar runtime efficiency and the same recommendation performance as the centralized bandit with a theoretical privacy-protection guarantee. We also point out that incorporating more user features indeed improves the performance of the contextual bandit in the vertical setting, which provides insight to improve the recommendation quality in an online and federated manner.

\section{Related Work}

\subsection{Federated Contextual Bandits}
Recently a few works have explored the concept of federated bandits \cite{shi2021federated, shi2021federated, agarwal2020federated, li2020federated, zhu2021federated}.
Among the federated bandits with linear reward structures, \citeauthor{DBLP:conf/nips/HuangWYS21} consider the case where individual clients face different K-armed stochastic bandits coupled through common
global parameters \cite{DBLP:conf/nips/HuangWYS21}.
They propose a collaborative algorithm to cope with the heterogeneity across clients without exchanging local feature vectors or raw data. 
\citeauthor{DBLP:conf/nips/DubeyP20} design differentially private mechanisms for tackling centralized and decentralized (peer-to-peer) federated learning \cite{DBLP:conf/nips/DubeyP20}.
\citeauthor{DBLP:books/sp/17/TewariM17} consider a fully-decentralized federated bandit learning setting with gossiping, where an individual agent only has access to biased rewards, and agents can only exchange their observed rewards with their neighbor agents \cite{DBLP:books/sp/17/TewariM17}. 
\citeauthor{DBLP:conf/aistats/LiW22} study the federated linear contextual bandits problem with heterogeneous clients respectively, i.e., the clients have various response times and even occasional unavailability in reality. They propose an asynchronous event-triggered communication framework to improve our method’s robustness against possible delays and the temporary unavailability of clients \cite{DBLP:conf/aistats/LiW22}.
\citeauthor{DBLP:conf/aistats/ShiSY21} consider a similar setting to
\citeauthor{DBLP:books/sp/17/TewariM17}'s setting but the final payoff is a mixture of the local and global model \cite{DBLP:conf/aistats/ShiSY21}. 
All the above-mentioned paper belong to the horizontal federated bandits, i.e., each client has heterogeneous users and all the clients jointly train a central bandit model without exchanging its own raw users' features. 
Since each user has her features all stored by one client, this horizontal federated setting has a sharp contrast with the VFL setting.
As far as we know, our paper is the first to consider the contextual bandit problem in the VFL setting.

\subsection{Vertical Federated Learning}
Vertical Federated Learning (VFL), aiming at promoting collaborations among non-competing organizations/entities with vertically partitioned data, has seen its huge potential in applications \cite{DBLP:journals/corr/abs-2202-04309} but still remains relatively less explored. 
Most of the existing work focuses on supervised learning and unsupervised learning settings.
For unsupervised learning problems,
\citeauthor{DBLP:conf/kdd/ChaiWZYC0022} propose the first masking-based VFL singular vector decomposition (SVD) method. 
Their method recovers to the standard SVD algorithm without sacrificing any computation time or communication overhead \cite{DBLP:conf/kdd/ChaiWZYC0022}.
\citeauthor{DBLP:journals/corr/abs-2203-01752} propose the federated principal component analysis and advanced kernel principal component analysis for vertically partitioned datasets \cite{DBLP:journals/corr/abs-2203-01752}.
For supervised learning problems,
\citeauthor{chen2020vafl} solves vertical FL in an asynchronous
fashion, and their solution allows each client to run stochastic gradient algorithms without coordination with
other clients. 
\citeauthor{hardy2017private} propose a VFL variant of logistic regression using homomorphic encryption \cite{hardy2017private}.
\citeauthor{wu2020privacy} propose a VF 
decision tree without dependence on any trusted third party \cite{wu2020privacy}.
\citeauthor{DBLP:conf/kdd/HuNYZ19} designed an asynchronous stochastic gradient descent algorithm to collaboratively train a central model in a VFL manner \cite{DBLP:conf/kdd/HuNYZ19}.
\citeauthor{DBLP:journals/corr/abs-2104-00489} introduce a framework to train neural network in the VFL setting using split neural networks \cite{DBLP:journals/corr/abs-2104-00489}.

\section{Preliminary}

\subsection{Problem Description}
We first introduce the definition of Linear Contextual Bandits.
\begin{definition}[Linear Contextual Bandits]
Consider a sequential decision process, where at each time step $t$, the environment would release a set of contexts $\{x_{t,a}\in \R^{d}\}_{a\in \mathcal{A}}$ for some fix $\mathcal{A}$ as the action set.
If context $x_{t,a_t}$ is been selected, then the environment would release a reward $r_t = x_{t,a_t}^{\top}\theta^{*}+\epsilon_t$ for some reward generating parameter $\theta^{*}\in \R^{d}$ and i.i.d. noise $\epsilon_t$.
A bandit algorithm $A$ is defined as a sequence of maps $\{\pi_{t}\}_{t=1}^T$.
Each map $\pi_{t}$ is from the historical information, $\{(x_{s,a_s}, r_s)\}_{s=1}^{t-1}$ and the current context $\{x_{t,a}\}_{a\in \mathcal{A}}$, to the distribution over the action set, $\Delta(\mathcal{A})$. 
The goal for the bandits algorithm is to minimize the long time regret, where regret is the gap between best possible achievable rewards and the rewards obtained by the algorithm. 
The regret is formally defined as:
\begin{align*}
    R(T) = \mathbb{E}[\sum_{t=1}^T r_{t,a_t^*}] - \mathbb{E}[\sum_{t=1}^T r_{t,a_t}],
\end{align*}
where $a_t^{*}=\arg\max_{a\in \mathcal{A}_t} x_{t,a_t}^{\top}\theta^{*}$ denotes the arm with the optimal reward.
\end{definition}

The task of personalized recommendation can be naturally modeled as a
contextual bandit problem. At a high level,
the contextual bandit aims to sequentially select arms (i.e., recommend items) to serve users when arrived in an online stream. The core of the bandit algorithm is an arm-selection strategy which is based on contextual information of the user (e.g., age and hobbies) and can be updated using  user-item historical feedback (click or not). The goal of the algorithm is to maximize the total reward (i.e., total user clicks) in the long run.  
In this paper, we consider the most widely studied linear contextual bandits, which is formally defined as follow.

\begin{figure}[!htbp]
    \includegraphics[width=0.5\textwidth]{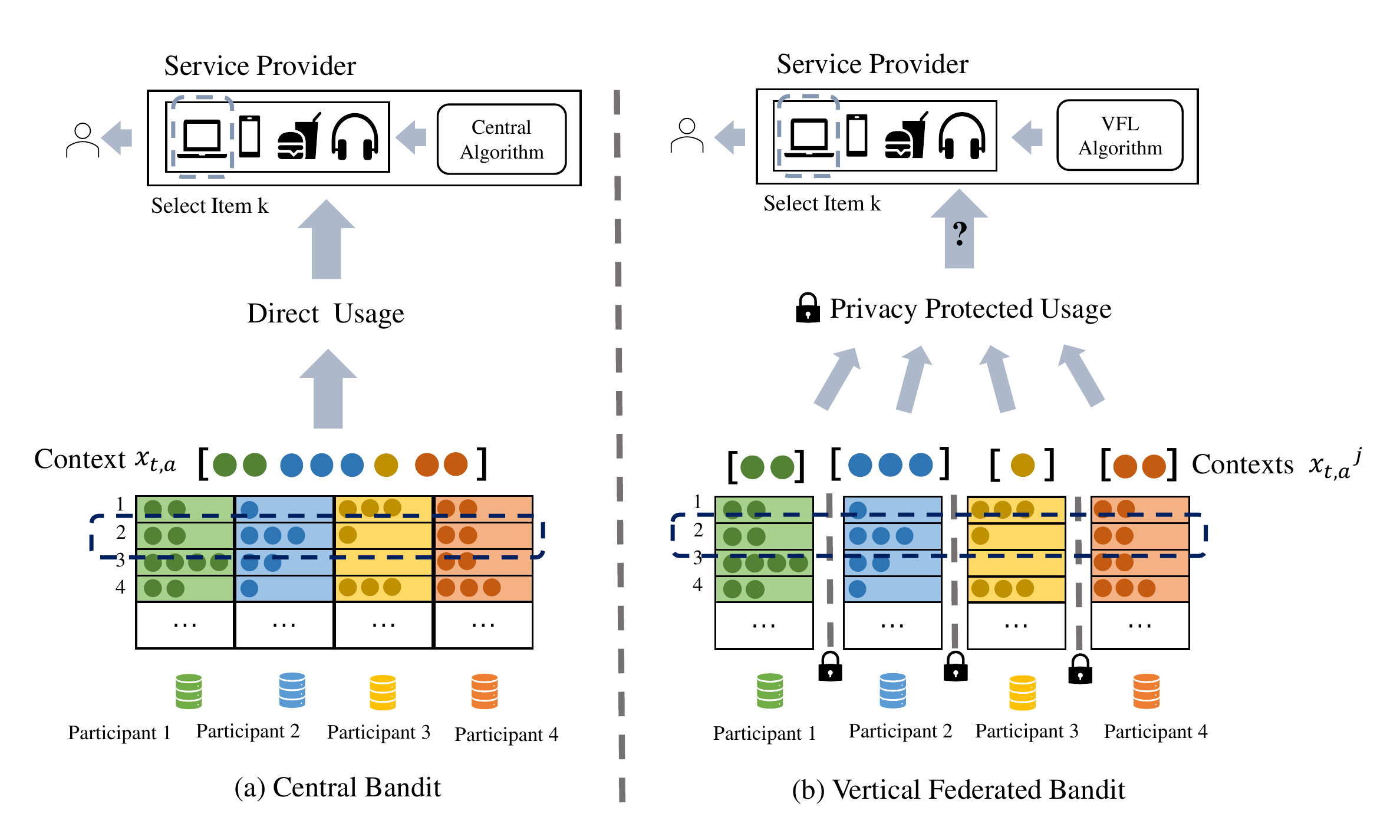}
    \centering
    \caption{Comparison between centralized(a) and vertical federated(b) setting.} %
    \label{fig:bandits-overview}
\end{figure}

Figure~\ref{fig:bandits-overview} shows a high-level comparison between centralized and vertical federated setting. 
In summary, user-item contextual information are collected by different participants and the figure uses different colors to distinguish the source of the information (e.g., information marked with green color are from participant 1). 
The number of circles denotes the specific feature value. 
Each row of the table contains the contextual information of one given user (denoted as $x_{t,a}$ ). 
The top box denotes the service provider, which is responsible for building contextual bandits for recommending items to users \footnote{In Figure~\ref{fig:bandits-overview}, the service provider is distinct from the data providers for simplicity. In practice, the service provider can also provide partial contextual information.}.
For centralized setting (Figure~\ref{fig:bandits-overview}a), service provider has access to all contextual information even though they belong to different participants. In contrast, in the vertical federated setting (Figure~\ref{fig:bandits-overview}b), the service provider cannot access the information directly from other participants due to privacy issues (there are locks between different columns in Figure~\ref{fig:bandits-overview}b). Therefore, these information need to be used in a privacy-preserving way which is the main goal of this paper. Specifically, we present an effective yet lossless way to use these information while protecting the users' privacy.  We will present our solution in Section~\ref{sec:algorithm}.

\subsection{Contextual Bandit in the Centralized Setting}

\label{sec:centralized-UCB}

We first introduce the most commonly-used contextual bandit algorithms, LinUCB and LinTS in the centralized setting, i.e., all user's contextual information can be fully accessed by the service provider.

In this setting, the recommendation service provider runs the linear contextual algorithm $A$ which proceeds in discrete trials indexed by time $t=1,2,3,\cdots, T$. 
In trial $t$, the learning procedure can be formulated as the following steps:
\begin{enumerate}[leftmargin=*]
    \item A observes an user $u_t$ and current arm set $\mathcal{A}_t=[K]$ consisting of $K$ different arms $a$ (e.g., items for recommendation). The \emph{context} $x_{t,a}$ can be seen as a feature vector describing the contextual information of both user $u_t$ and the given arm $a$.
    \item $A$ chooses arm based on the estimated value $v_{t,a}$ for each each arm 
    \begin{align}
        \label{eq:UCB}
        a_t = \argmax_{a\in \mathcal{A}_t} v_{t,a}.
    \end{align}
    For LinUCB, detailed in the left part of Figure~\ref{fig:bandits-overview}, $v_{t,a}$ is called UCB value with definition $v_{t,a} = \hat{r}_{t,a} + \hat{B}_{t,a}$, where $\hat{r}_{t,a} \coloneqq x_{t,a}^{\top}\hat{\theta}_t $ is the estimated mean for item $a$ and \\ $\hat{B}_{t,a} \coloneqq \beta_t \sqrt{x_{t,a}^{\top} \Lambda^{-1}_t x_{t,a}}$ is the optimistic term to encourage exploring new items. In particular, \citeauthor{DBLP:conf/nips/Abbasi-YadkoriPS11} prove that the true reward of arm $a$ is less than the $\hat{r}_{t,a} + \hat{B}_{t,a}$ with high probability, thus $\hat{r}_{t,a} + \hat{B}_{t,a}$ is an optimistic estimation for the true reward of arm $a$ for some appropriate $\beta_t$ \cite{DBLP:conf/nips/Abbasi-YadkoriPS11}.
    Here $\Lambda_t \coloneqq \sum_{s=1}^{t-1} x_{s,a_s}x_{s,a_s}^{\top}+\lambda \cdot I$ for some $\lambda>0$. 
    As for LinTS, $v_{t,a} = x_{t,a}^{\top}\hat{\mu}_{t}$ where $\hat{\mu}_{t}\sim \mathcal{N}(\hat{\theta}_t, v^2 \Lambda_t^{-1})$ with $v^2$ is the variance of the reward \cite{agrawal2013thompson}.
    The sampling randomness in the $\tilde{\Lambda}_{t}$ encourages the algorithm to explore different arms.
    \item Once $A$ recommends $a_t$ for user $u_t$, then it receives a random reward $r_{t} = x_{t,a_t}^{\top}\theta^{*}+\epsilon_t$ where $\{\epsilon_t\}_{t=1}^{T}$ are i.i.d. sub-Gaussian mean zero noise \cite{wainwright2019high}. This is demonstrated in the right part of Fig.~\ref{fig:bandits-overview}.
    Here, $\theta^{*}\in \R^d$ is an unknown parameter used for generating rewards. 
    $\hat{\theta}_t$ is a known parameter for predicting reward and the goal of $A$ is to approximate $\theta^*$ with $\hat{\theta}_t$.
    \item Then A uses the newly observed reward $r_t$ and the new sample covariance matrix $\Lambda_{t+1}$ to update the estimator $\hat{\theta}_t$ via the Ordinary Least Square (OLS) algorithm, i.e., $\hat{\theta}_{t+1} = \Lambda_{t+1}^{-1} u_{t+1}$ where $\Lambda_{t+1} = \Lambda_{t} + x_{t,a_t}x_{t,a_t}^{\top}$ and $u_{t+1} = u_{t} +x_{t,a_t}r_t$.
\end{enumerate}

\subsection{Vertical Federated Contextual Bandits}
\label{sec:VFL-settings}
In this subsection, we formally present the VFCB problem, upon which we will design our privacy-protected algorithms.

\nosection{Participant Role}
In our VFCB setting, there is a single \emph{active} participant (AP) and multiple \emph{passive} participants (PP).
Specifically, active participant  directly interacts with users and holds historic user feedback (e.g., click an ad or not in the recommendation system). 
This participant is always the recommendation service provider.  
In contrast, passive participants can be seen as user feature providers who hold complementary
contexts with respective user. 
They do not directly interact with user.
Without loss of generality, we assume the total number of participants are $M$, and we index the participants as $1,2,\cdots, M$.
Here the participant $1$ is the AP and the remaining ones are passive. 
Additionally, a privacy mask generator(PMG) is required to perform our O3M protocols. 
This can either be a secure third party or been randomly selected from PP with some mechanism agreed upon.

\nosection{Vertical Federated Setting}
In contrast to the centralized setting, the vertical federated setting refers to that the user contexts ($x_{t,a}, a\in \mathcal{A}$)  are distributively stored in different departments and cannot be shared since privacy concern or data regulations. 
Formally speaking, a specific context $x_{t,a}$ is divided into different parts $x_{t,a}=[x_{t,a}^{1},x_{t,a}^{2},\cdots,x_{t,a}^{M}]$ and the $j$-th local context $x_{t,a}^{j}\in \R^{d\times d_j}$ can only be accessed by the $j$-th participant. 
In this setting, the core task is how to jointly leverage these local contextual information to approximate or even recover the reward mean and the upper-confidence bound in the centralized setting, i.e., presented in Eq~\ref{eq:UCB}. 
As introduced in Section \ref{subsec:masked}, this paper presents a simple yet effective way to achieve this task, which is illustrated in Figure~\ref{fig:masking-process}.

\nosection{Threat Models}
We consider all parties involved in VFCB semi-honest in our paper. This is inline with previous work on vertical federated learning \cite{DBLP:journals/tist/RenYC22}.The semi-honest model is not as strong as other threat model but nonetheless it is still powerful for our inter-departmental vertical federated learning setting. In semi-honest setting, each participant adheres to the designed protocol but it may attempts to infer information about other participant's input (i.e., local contextual information). Additionally, our threat model assumes that our PMG cannot collude with AP under any circumstances. This ensures the privacy constraint can be achieved via following privacy analysis.
\section{Algorithm}
\label{sec:algorithm}
In this section, we present our O3M privacy technique and then demonstrate how to apply it in the linear contextual bandits model to encrypt the user's information.

\subsection{Challenges}
As shown by the previous literature, LinUCB and LinTS algorithm can both achieve rate-optimal regret for the linear contextual bandits problem. 
In this section, we design a privacy-protect variant of them to adopt to the vertical federated setting while perfectly recovering to the centralized one. 
In fact, developing vertical federated variants for them are nontrivial.
As illustrated in Figure~\ref{fig:bandits-overview}, due to the vertical federated setting, each participant only have access to the private version of the contexts from other participants. 
Typically each participant can apply the general widely-used privacy-awareness technique, e.g., differential privacy (DP) or multi-party computation (MPC), to protect their context before sending out to others.
However, DP has been documented to significantly degrade the performance of the contextual bandits \cite{DBLP:conf/nips/ShariffS18} while MPC would add large amount of extra computational cost.

To design a VFCB framework without the above drawbacks, we have to step into the key quantities of the LinUCB/LinTS algorithm.
For LinUCB, firstly, both the estimated reward and the optimistic terms $\hat{r}_{t,a}$ and $\hat{B}_{t,a}$ require the $\Lambda_t$ in their computation.
Using private contexts $\tilde{x}_{t,a_t}$ to construct the private version of $\Lambda_t$, denoted as $\tilde{\Lambda}_t$, might introduce extra bias.
For $\hat{r}_{t,a}$, $\tilde{\Lambda}_{t}^{-1}$ then would be multiplied with the $\sum_{s=1}^t \tilde{x}_{s,a_{s}}\tilde{r}_s$, where $\tilde{x}_{s,a_{s}}$ and $\tilde{r}_s$ are both (possibly) private and different from the non-private one.
Once the $\tilde{\theta}_t$ is obtained via the private $\tilde{\Lambda}_t^{-1}$ and private historical rewards, it then be utilized to multiply with the private context $\tilde{x}_{t,a}$ to construct the estimated reward for arm $a$ at time $t$.
For $\hat{B}_{t,a}$, $\tilde{\Lambda}_h^{-1}$ is utilized to multiply with the private context to derive the optimistic term $\tilde{x}_{t,a}^{\top}\tilde{\Lambda}_h^{-1}\tilde{x}_{t,a}$. 
For LinTS, it shares the most of the computation steps as LinUCB, especially the sample covariance matrix $\tilde{\Lambda}_t$.

In short, to maintain a small gap between the vertical federated bandits and the centralized ones, we need to design a sound privacy-protection mechanism for the contexts as well as the rewards by exploiting the computation details described above.
The target of the privacy-protection mechanism is to keep the estimator $\tilde{\theta}_t$ and $\tilde{\Lambda}_t$ derived from the private contextual information close to the non-private one without too much extra computational cost.

\subsection{O3M}
\label{subsec:masked}
Motivated by prior study \cite{DBLP:conf/kdd/ChaiWZYC0022} on masking data for vertical federated single value decomposition, we propose O3M on contextual bandits.

From a high level perspective, O3M allocates each participant a part of a pre-defined orthogonal matrix to mask their contexts and send the masked contexts to the AP. 
Then AP sums the masked local contexts to derive the masked contexts and directly run the LinUCB/LinTS upon them.We illustrate the O3M step by step with Figure~\ref{fig:masking-process}. 
To help the presentation, we use $d_j$ to denote the dimension for the local contextual information $x_{t,a}^{j}$ stored in participant $j$. 
The total dimension for the global contextual information $x_{t,a}$ is $d=\sum_{j=1}^M d_j$. 
\footnote{ $Q \in R^{d \times d}$ should be a square as $\sum d_j = d$. However, mask $Q$ and subsequent $Q^j$s are drawn not to scale in Figure~\ref{fig:masking-process} for illustration purpose. } 

\begin{figure}[t]
    \includegraphics[width=0.5\textwidth]{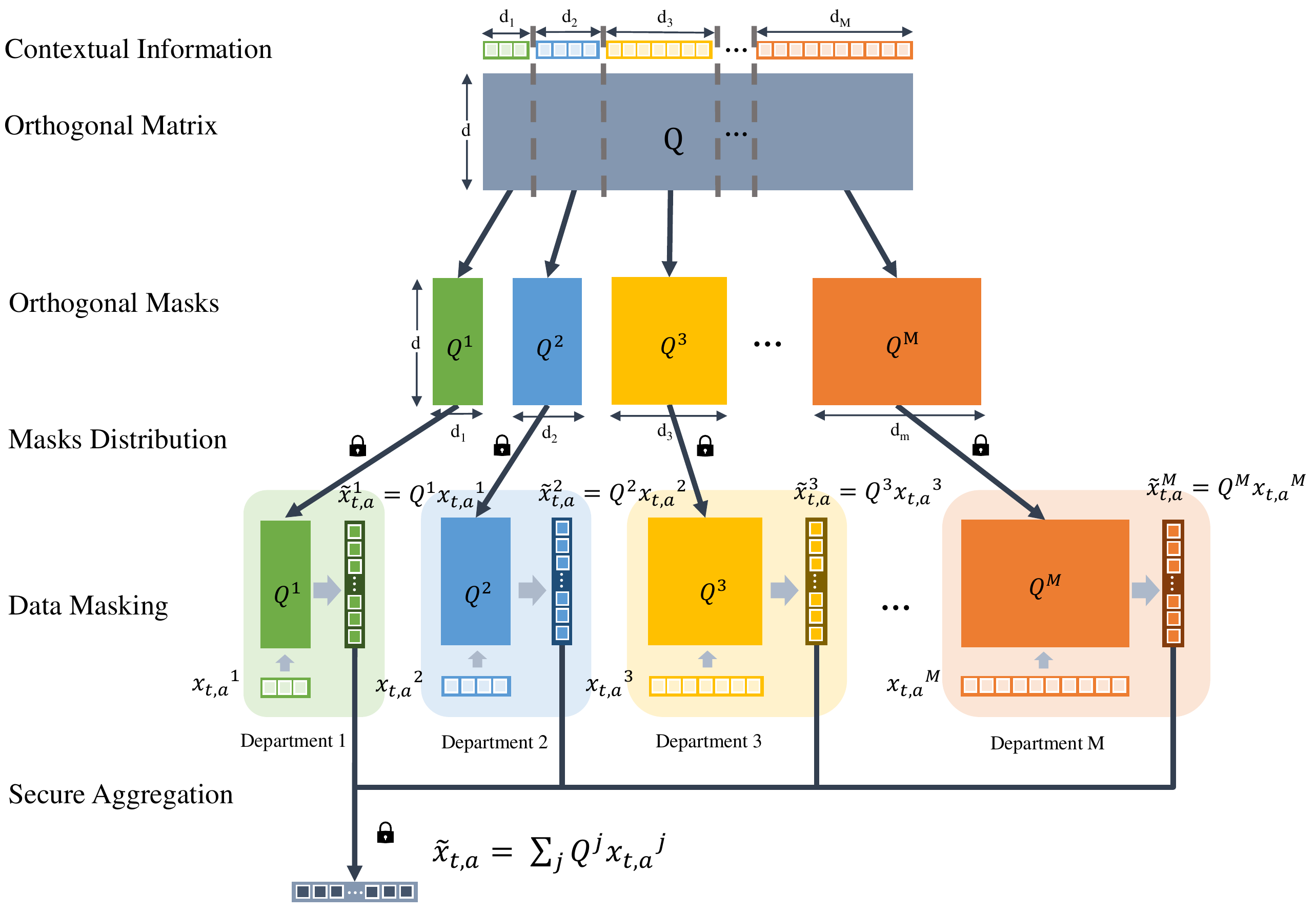}
    \centering
    \caption{Overview for O3M process in VFCB. In PMG, the orthogonal mask $Q \in R^{d \times d}$ is partitioned with context columns $d_j$ into $Q^j \in R^{d \times d_j}$.  The mask $Q_j$ is then distributed to each participant $j$ securely. After that, each participant conducts masking process, producing $\tilde{x}_{t,a}^j = Q_j x_{t,j}$. Finally, the data is securely aggregate as $\tilde{x}_{t,a} = \sum_j Q_j x_{t,j}$. }
    \label{fig:masking-process}
\end{figure}

At the beginning, either a third party or a PP randomly selected by some mechanism is agreed by all the participants as the Private Mask Generator (PMG). 
Then the PMG will generate an orthogonal matrix $Q\in \R^{d\times d}$. 
The orthogonal mask $Q$ is partitioned by columns into $M$ parts to match the dimension $d_j$ of each participant's context $x_{t,a}^{j}$, i.e., $Q=[Q^1,Q^2, \cdots, Q^M]$ such that $Q^j\in \R^{d\times d_j}$.
Each participant's mask is distinguished by colored solid squares (e.g. solid green mask $Q^1$ is dedicated for participant 1).
The mask $Q^j$ is then securely sent to participant $j$ (the lock on each arrow indicates the secure peer to peer transfer process).

After receiving the masks, each participant would utilize it to encrypt her local contexts.
The details of encryption can be seen at the lower section of Figure~\ref{fig:masking-process} and 
we use different colors to distinguish contexts for different participants (following the same color scheme with the partitioned mask $Q^j$).
The data provider $j$ obtains the masked context $\tilde{x}_{t,a}^{j}$ by multiplying its mask $Q^j$ (solid square) with its local context $x_{t,a}^{j}$ (horizontal hollow bars) $\tilde{x}_{t,a}^{j}$ (vertical hollow bars).
Note that after the masking, the masked context $\tilde{x}_{t,a}^{j}$ shares the same shape as the global context $\tilde{x}_{t,a}$.
The data provider then sends her masked contexts to the AP and the AP obtains the masked context $\tilde{x}_{t,a}$ by directly adding the masked local contexts together, following  $\tilde{x}_{t,a} = \sum_{j=1}^M \tilde{x}_{t,a}^{j}$.
Then the AP runs the LinUCB/LinTS protocol on the private contexts $\tilde{x}_{t,a}$ for $a\in \mathcal{A}_t$. 
We summarize the above procedure into the VFUCB in algorithm~\ref{alg:vfUCB} and the corresponding VFTS is deferred in Appendix~\ref{alg:VFTS}

\begin{algorithm}[t]
\caption{VFUCB}
\label{alg:vfUCB}
\begin{algorithmic}[1]
  \STATE \textbf{Init:} $\tilde{\Lambda}_1 = \mathbf{I}_d$, $\tilde{u}_1 = \mathbf{0}$
  \STATE Any PP or a secure third party randomly generates a orthogonal $Q$ and conduct a column partition $Q=[Q^1,Q^2, \cdots, Q^M]$ where $Q^i\in \R^{d\times d_i}$ and send the i-th submatrix to the client i
    
  \FOR {$t=1$ {\bfseries to} $T$}
  
  \STATE An action set $X_t=\{x_{t,i}\}_{i=1}^K$ arrives to the system while the client $j$ can only observe $x_{t,i}^{j}\in \R^{d_j\times 1}$ and $x_{t,i} = [x_{t,i}^{1}; x_{t,i}^{2}; \cdots; x_{t,i}^{M}]$

  \STATE For each $j\in[2,\cdots,M]$, Client j send $\{Q^j x_{t, i}^{j}\}_{i\in[K]}$ to the AP to derive $\tilde{x}_{t,i}=\sum_{j=1}^M Q^j x_{t,i}^{j}$
  
  \STATE AP recommend the $a_t\in[K]$ item via 
  \begin{align*}
  \vspace{-1ex}
      a_t = \arg \max_{a\in \mathcal{A}} \tilde{r}_{t,a} + \tilde{B}_{t,a} 
  \end{align*}
  where 
  \begin{align}
  \vspace{-1ex}
      \tilde{r}_{t,a} \coloneqq \tilde{x}_{t,i}^{\top}\tilde{\theta}_t, \quad \tilde{B}_{t,a} \coloneqq \beta_t \sqrt{\tilde{x}_{t,i}^{\top} \tilde{\Lambda}_{t}^{-1} \tilde{x}_{t,i}}.
  \end{align}

  \STATE AP receives $r_t$ from the user

  \STATE AP update $\Lambda_{t+1}$ via 
      $\tilde{\Lambda}_{t+1} = \tilde{\Lambda}_{t} + \tilde{x}_{t,i}\tilde{x}_{t,i}^{\top}$
  \STATE Update $\tilde{u}_{t+1}$ via 
      $\tilde{u}_{t+1} = \tilde{u}_{t} + r_t \tilde{x}_{t,i}$

  \STATE AP update $\tilde{\theta}$ via 
      $\tilde{\theta}_{t+1} = \tilde{\Lambda}_{t+1}^{-1}  \tilde{u}_{t+1}.$
  \ENDFOR 
\end{algorithmic}
\end{algorithm}

\subsection{Lossless Recommendation}
\label{subsec:lossless}
In this subsection, we will prove that our VFUCB/VFTS algorithm running on the private data can behave perfectly the same as the centralized ones using on the non-private data.
\footnote{For demonstration purpose, we focus on VFUCB for lossless recommendation analysis. The analysis for VFTS can be find in Appendix~\ref{appendix:VFTS}.}

We provide the theoretical justification to show that, although the masked contexts are biased from the original one, all the bias will be canceled out in the estimators constructed from the masked contexts and the estimators can perfectly recover to the non-private ones.
The key observations of our analysis are two-folds:
The first is that for any matrix $Q\in \R^{d\times d}$, any vector $x\in \R^{d\times 1}$, with any column partition $Q=[Q^1, Q^2, \cdots, Q^M]$ where $Q^j\in \R^{d\times d_j}$ and $x=[x^{1}, x^{2}, \cdots, x^{M}]$ where $x^{j}\in \R^{d\times d_j}$, we have $Q x = \sum_{j=1}^M Q^j x^{j}$.
Thus, the summation of all the masked contexts $\sum_{j=1}^{M}Q^j x_{t,a}^{j}$, is exactly the transformed context with the orthogonal matrix $Q$ as the transformer, i.e., $Qx_{t,a}$.

Thus, the AP is in fact running a centralized linear contextual bandits algorithms on the private contexts $\tilde{x}_{t,a} = Qx_{t,a}$. 
We will prove the high privacy protection level using the O3M, i.e., the masked contexts can hardly reveal the information about the original contexts.
While the gap between masked contexts and original ones are inevitable in safeguarding the privacy,
we need to ensure the statistics utility is not severely damaged so that our VFCB/VFTS can still achieve the rate-optimal regret bound.
Here comes to our second observation: by choosing the mask matrix $Q$ as any orthogonal matrix in our O3M, our VFUCB/VFTS can lossless recover to the non-private LinUCB/LinTS algorithms.

To be more concrete, taking LinUCB as an example, the decision is made by choosing the item with the largest UCB values, summation of privated predicted reward $\tilde{r}_{t,a}$ and the optimistic term $\tilde{r}_{t,a}$.
Thus we will prove that $\tilde{r}_{t,a}$ and $\tilde{B}_{t,a}$ are exactly the same as the estimated rewards $\hat{r}_{t,a}$ and optimistic terms $\tilde{B}_{t,a}$ constructed in the centralized LinUCB bandits.
We formalized our observations in the following theorem.

\begin{theorem}
Given the fixed sequence $\{x_{t,a}\}_{t\in[T], a\in A}$ and corresponding return sequence $\{r_{t,a}\}_{t\in[T], a\in A}$, for any time $t\in [T]$, we have:
\vspace{-1ex}
\begin{itemize}[leftmargin=*]
    \item The estimated value obtained by LinUCB (\ref{eq:UCB}) $\hat{r}_{t,a}$, is the same as the one obtained by the VFUCB (Algorithm~\ref{alg:vfUCB}) $\tilde{r}_{t,a}$.
    \item The optimistic value obtained by LinUCB (\ref{eq:UCB}) $\hat{B}_{t,a}$, is the same as the one obtained by the VFUCB (Algorithm~\ref{alg:vfUCB}) $\tilde{B}_{t,a}$.
\end{itemize}
\end{theorem}

\begin{proof}
\vspace{-1ex}
Our proof is completed using the mathematical induction:
\begin{enumerate}[leftmargin=*]
    \item For $t=0$, we have $\hat{\theta}_0 = \tilde{\theta}_0 = \bm{0}$ and $\hat{\Lambda}_0 = \tilde{\Lambda}_0 = \lambda I$ following the LinUCB and VFUCB. 
    For any $a\in \mathcal{A}_t$,
    \begin{align*}
        \tilde{B}_{0,a}&= \beta_0\sqrt{\tilde{x}_{t,a}^{\top}\tilde{\Lambda}_0^{-1}\tilde{x}_{t,a}} = \beta_0\sqrt{\lambda^{-1} x_{t,a}^{\top}Q^{\top}QQ^{\top}Qx_{t,a}}\\
        &= \beta_0\sqrt{\lambda^{-1} x_{t,a}^{\top}x_{t,a}} = \beta_0 \sqrt{x_{t,a}^{\top}\Lambda_0^{-1}x_{t,a}} = \hat{B}_{0,a}.
    \end{align*}
    \item Assume $\hat{\theta}_s = \tilde{\theta}_s$ and $\hat{B}_{s,a} = \tilde{B}_{s,a}$ for $s\in[t-1]$ and all $a\in \mathcal{A}_s$. 
    Since the decision at any time $s$ is determined by the estimated values and optimistic terms for each arm, the historical selected actions by the LinUCB are the same as those of the VFUCB.
    Now we consider the case $t$.

    We first prove the relationship between $\tilde{\Lambda}_t$ and $\Lambda_t$.
    \begin{align*}
            \tilde{\Lambda}_t^{-1} &= \left(\lambda I + \sum_{s=1}^{t-1} \left (\sum_{j\in[M]}Q^jx_{t, a_t}^{j}\right) \left(\sum_{j\in[M]}Q^jx_{t, a_t}^{j}\right)^{\top}\right)^{-1}\\
            &= \left (\lambda I + \sum_{s=1}^{t-1} Qx_{t,a_t}(Qx_{t,a_t})^{\top}\right)^{-1}\\
            &= \left (\lambda I + \sum_{s=1}^{t-1} Qx_{t,a_t}x_{t,a_t}^{\top}Q^{\top}\right)^{-1} = \left (Q\Lambda_{t}Q^{\top}\right)^{-1} = Q\Lambda_t^{-1}Q^{\top}.\\
    \end{align*}
    Thus for any $t\in [T]$ and $a\in \mathcal{A}_t$, we first prove that $\hat{B}_{t,a} = \tilde{B}_{t,a},\hat{B}_{t,a} = x^{\top}_{t,a} \Lambda_t^{-1} x_{t,a} = x^{\top}_{t,a} Q^{\top} Q \Lambda_t^{-1} Q^{\top} Q x_{t,a} = \tilde{x}^{\top}_{t,a} \tilde{\Lambda}_t^{-1} \tilde{x}_{t,a} = \tilde{B}_{t,a}.$ 
    Next we verify that $\hat{r}_{t,a} = \tilde{r}_{t,a}$, we decompose into the following three steps to finish.
    As we know, $\hat{r}_{t,a} = \hat{\theta}_{t}^{\top}x_{t,a}$ and thus we first investigate the relationship between $\hat{\theta}_{t}$ and $\tilde{\theta}_t$.
    Then,\\$
        \tilde{\theta}_t = \tilde{\Lambda}_{t}^{-1}  \tilde{u}_{t} = Q\Lambda_t^{-1}Q^{\top} Q (\sum_{s=1}^t r_s x_{s,a_s}) = Q\Lambda_t^{-1} (\sum_{s=1}^t r_s x_{s,a_s}) = Q\hat{\theta}_t.$
    Finally, $
        \tilde{r}_{t,a} = \sum_{j=1}^M (Q^{j}x_{t,a}^{j})^{\top}\tilde{\theta}_t = (Qx_{t,i})^{\top}\tilde{\theta}_t= x_{t,i}^{\top}Q^{\top}Q\hat{\theta}_t = x_{t,i}^{\top}\hat{\theta}_t = \hat{r}_{t,a}.$
    Now we prove that $\tilde{r}_{t,a} = \hat{r}_{t,a}$ and $\tilde{B}_{t,a} = \hat{B}_{t,a}$ for all $a\in \mathcal{A}_t$ and by the mathematical induction, this conclusions holds for all $t\in[T]$. 
\end{enumerate}
\end{proof}

\subsection{Complexity Analysis}
We present the complexity analysis for VFCB in terms of computation and communication cost in Table \ref{tab:complexity-analysis}. We follow the notation introduced in previous section. The detail derivations are presented in Appendix~\ref{appendix:complexity}. 
\begin{table}[ht]
\caption{Computational/Communication Complexity Analysis for VFCB}
\label{tab:complexity-analysis}
\begin{tabular}{@{}lll@{}}
\toprule
Model  & computation cost                   & communication cost                \\ \midrule
VFUCB  & $O(T \times (K\times M \times d + K \times d^2 + d^3))$ & $O(T \times K \times M \times d)$ \\
LinUCB & $O(T \times (K \times d^2 + d^3))$                                   & \multicolumn{1}{c}{N/A}           \\ \midrule
VFTS   & $O(T \times (K\times M \times d + K \times d^2 + d^3))$                                   & $O(T \times K \times M \times d)$ \\
LinTS  & $O(T \times (K \times d + d^3))$ & \multicolumn{1}{c}{N/A}           \\ \bottomrule
\end{tabular}
\vspace{-2ex}
\end{table}

As shown in the table, O3M introduces a common cost of $O(K\times M\times d + K \times d^2)$ for computing the mask. This is the major change in computational cost in our VFUCB protocol. On the other hand, for VFTS, as sampling only takes once per step $t$, the lower cost for calculating $a_t$ results in a complexity of only $(K\times d)$. Hence the O3M process contributes more and results in $O(K\times d^2)$ in VFTS. However, as we will discuss in Section~\ref{sec:experiment-answer}, in practice, the complexity of VFTS will be manageable under realistic settings of variables. As of communication cost, our communication complexity is linear towards the respective variables $T,K,d,M$.

\subsection{Privacy Analysis}
To show that O3M is indeed secure in the threat model, we need to prove that the masked data $\tilde{x_{t,a}^j}$ cannot be recovered without knowing the mask $Q^j$. This is achieved via the following theorem\footnote{The proof of the theorem can be find in Appendix~\ref{appendix:privacy}.}:

\begin{theorem}
\label{theorem:masked-secure}
Given a masked data $D = Q_{1}X_{1}$, there are infinite number of raw data $X_{2}$ that can be masked into $D$, i.e. $Q_{2}X_{2} = Q_{1}X_{1} = D$
\end{theorem}

This means that even AP knows all the masked data, given that it does not know any mask (other than the one associate with itself), since there are infinite number of mask and data combinations, it's not possible to directly infer PP's original user context data. Hence the user context data privacy is preserved within each PP. 
\section{Experiment}

In this section, we design experiments based on both simulation and real-world datasets to answer the following questions:

\begin{enumerate}[topsep=1ex,leftmargin=*]
    \item What is the performance gap between VFCB with privacy protection and the centralized bandits without privacy protection?
    \item What is the extra computation and communication cost for the VFCB compared with their centralized  variant?
    \item Whether more data introduced from other participants improves the performance of VFCB?
\end{enumerate}

In summary, Q1 and Q2 address the cost of safeguarding privacy in vertical federated settings while Q3 depicts the improvements from joint training using different data sources. In the following subsections, we first describe the settings for both synthetic and real-world datasets. Then, we use the experiment results to answer the three question proposed. For consistency, we adopt the same notation from previous sections. Details on experiments environment can be find at Appendix~\ref{subsec:experiment-env}.

\subsection{Experiment Settings}
\nosection{Synthetic Experiment}
In our synthetic dataset, $\forall t \in[T]$ and $\forall a \in[K]$, the context $x_{t,a}$ is generated from a multivariate normal distribution $x_{t,i} \sim \mathcal{N}(0,\sigma^2 I)$ with $\sigma^2 = 0.05$. It is then normalized with  l2 norm. The reward generating parameter $\theta$ is generated from a multivariate normal distribution $\theta \sim \mathcal{N}(0,\sigma^2 I)$ with $\sigma^2 = 0.05$. It is also normalized with l2 norm. The reward $r_t$ is then generated as $r_t = x^T_{t,a} \theta + \epsilon$ where $\epsilon \sim \mathcal{N}(0,0.05)$. The regret is then measured with $R(T) = \sum^{T}_{t=1} x^T_{t,a^{*}} \theta -x^T_{t, a_t} \theta$. The details for the experiment runs can be find at Appendix~\ref{subsec:synthetic-setting}.

\nosection{Criteo dataset}
Criteo dataset \cite{criteo_ai_labs_2014} is an ad stream log dataset provided by Criteo AI Lab on Kaggle Ad display challenge. Since the dataset is anonymized and the label meanings are unknown, we follow previous approach and assumptions from \cite{DBLP:conf/mlsys/MalekzadehAHL20} to construct the experiment dataset. The details for the experiment runs can be find at Appendix~\ref{subsec:criteo-setting}.

\subsection{Experiment Discussion}
\label{sec:experiment-answer}
We demonstrate our experimental results on Figure~\ref{fig:Synthetic-Analysis}, Figure~\ref{fig:Complexity-Comparison} and Figure~\ref{fig:Partial-Comparison}. We then use these results to answer in details the 3 questions we asked previously.
\begin{figure*}[t]
    \centering
    \subfloat[\label{regret-comparison-vfucb}]{\includegraphics[width=0.25\textwidth]{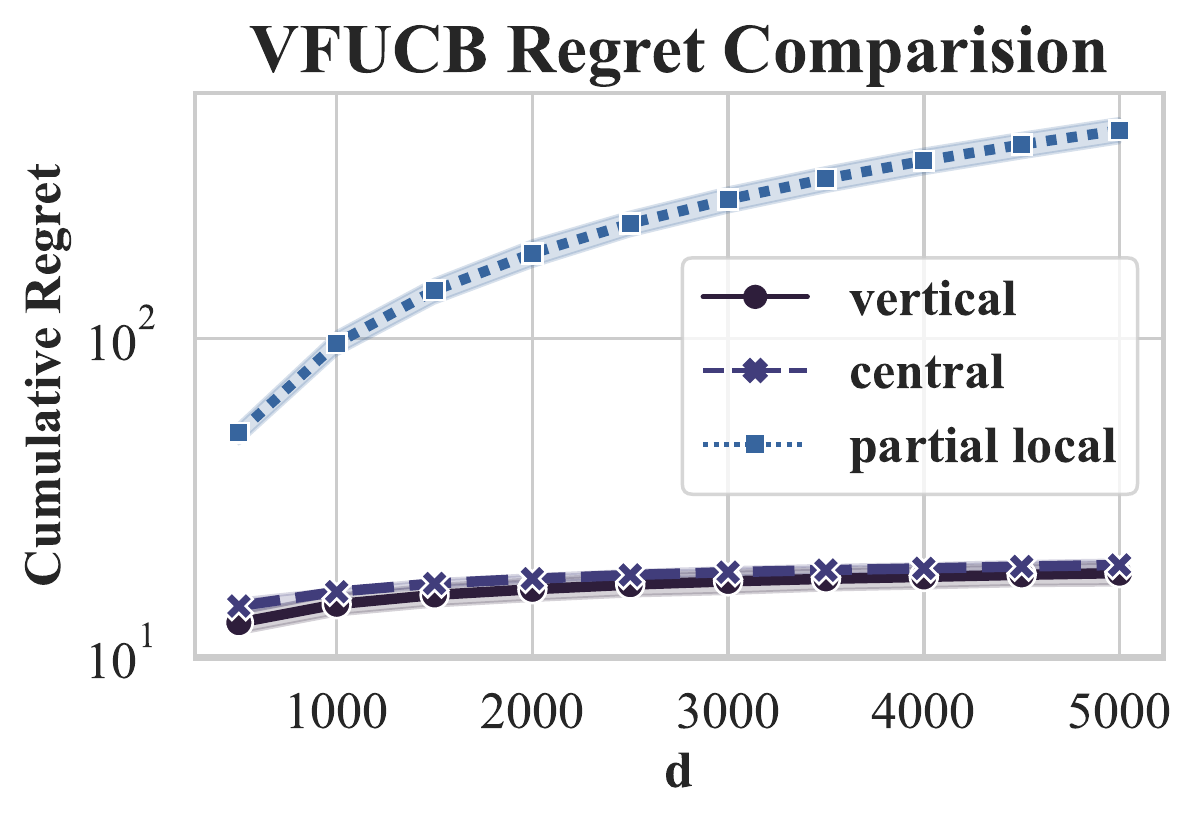}}
    \subfloat[\label{reconstruction-error-vfucb}]{\includegraphics[width=0.25\textwidth]{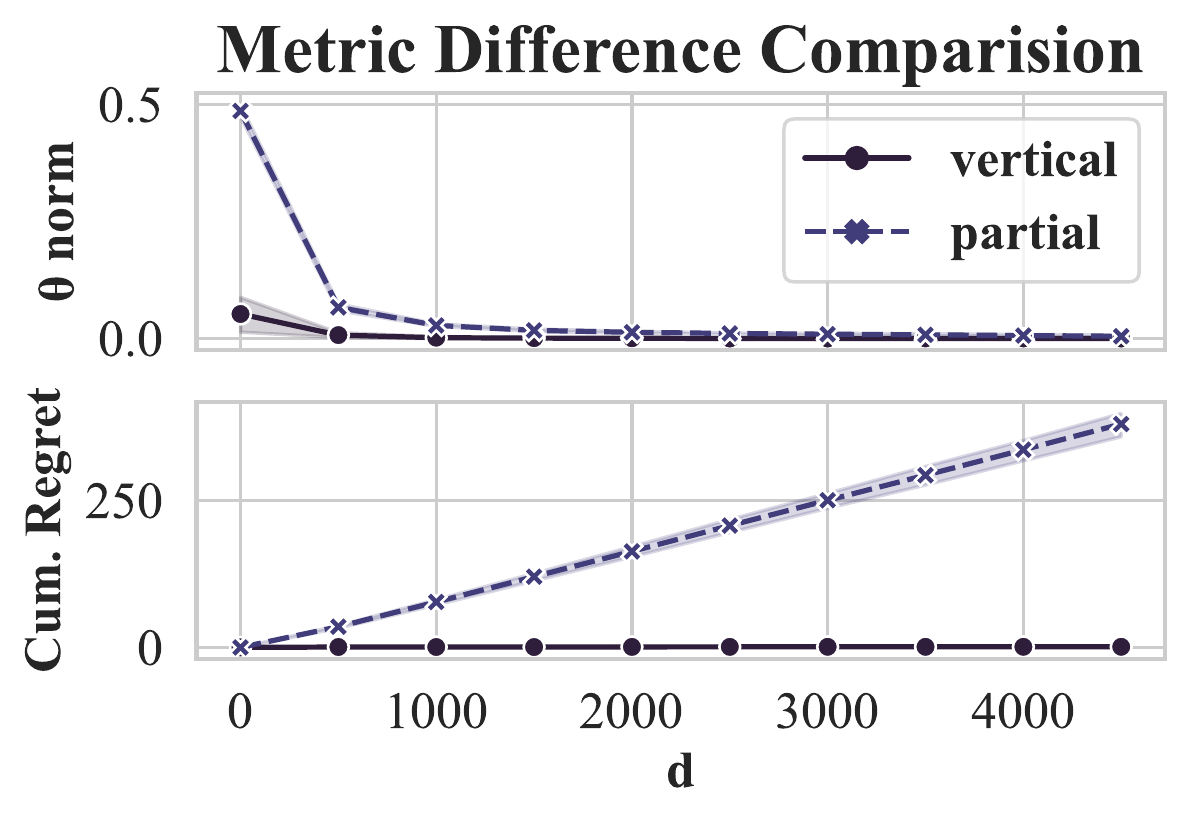}}
    \subfloat[\label{regret-comparison-vfts}]{\includegraphics[width=0.25\textwidth]{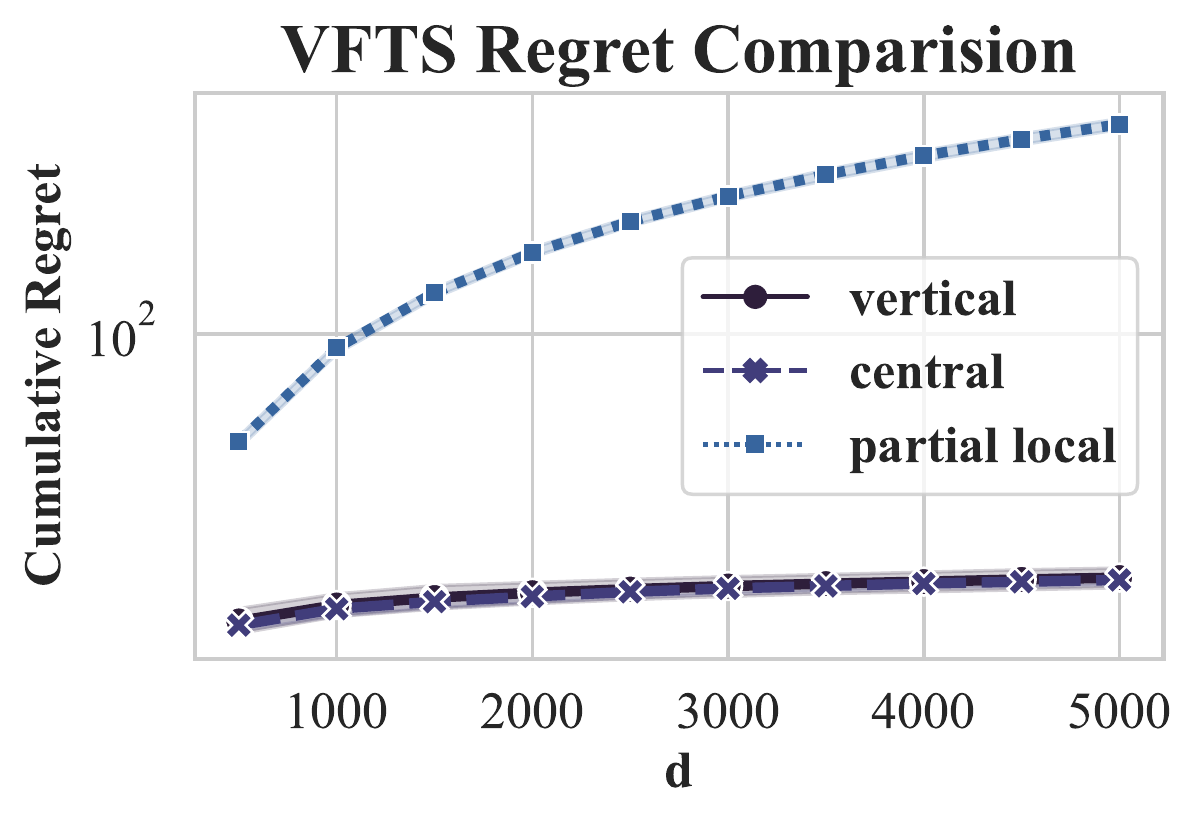}}
    \subfloat[\label{reconstruction-error-vfts}]{\includegraphics[width=0.25\textwidth]{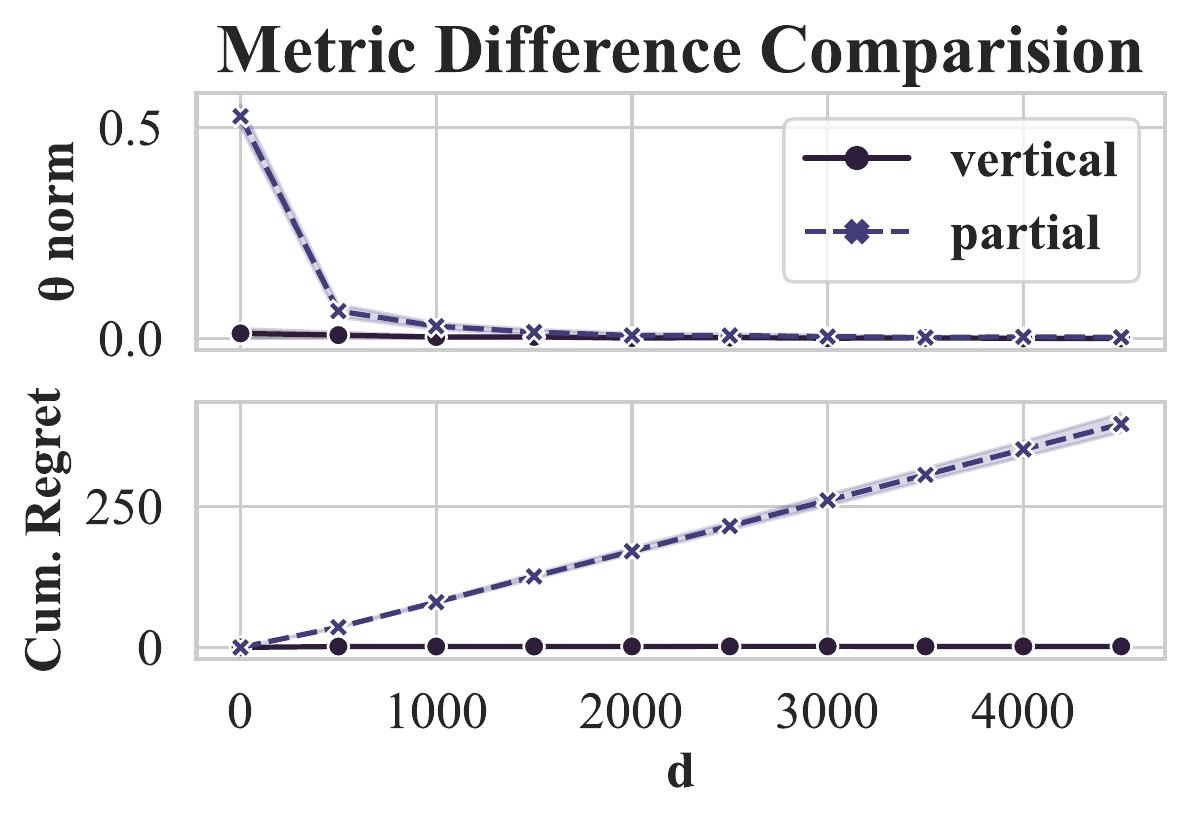}}
    \vspace{-2ex}
    \caption{Experimental results for (a) Comparison between cumulative regret of central LinUCB, partial-local LinUCB and VFUCB on synthetic dataset. (b) Comparison in metrics difference for VFUCB and partial-local LinUCB with Central LinUCB. (c) Comparison between cumulative regret of central LinTS, partial-local LinTS and VFTS on synthetic dataset. (d) Comparison in metrics difference for VFTS and partial-local LinTS with Central LinTS. For ease of comparison, the y scale for (a) and (c) is changed to log scale to differentiate between VFCB and partial bandits.}
    \label{fig:Synthetic-Analysis}
\end{figure*}

\nosection{The performance gap between VFCB with privacy protection and that of without privacy protection}
\label{nosec:answer-1}
Although we already establish the theoretical guarantee that our VFCB is lossless in Section~\ref{subsec:lossless}, we want to use experimental results to demonstrate it. Furthermore, we want to show that VFCB protocols can make the same decision as their centralized variant on experimental data. The result of our synthetic analysis is demonstrated on Figure~\ref{fig:Synthetic-Analysis}.

As shown in Figure~(\ref{regret-comparison-vfucb}), our VFUCB protocol completely achieves the same regret of central LinUCB. However, partial LinUCB cannot perform as well as VFCB. The regret of it grows significantly compared with VFUCB and full LinUCB, resulting in more than 10 times larger cumulative regret.  Similarly, for Figure~(\ref{regret-comparison-vfts}), our VFTS protocol achieves similar cumulative regret as central LinTS while the cumulative regret for partial LinTS grows to more than 10 times larger.

We further investigate the reconstruction error of the VFCB and particial bandits in Figure~(\ref{reconstruction-error-vfucb}) and Figure~(\ref{reconstruction-error-vfts}). In particular, we plot  $\lVert\widehat{\theta}_t^{\operatorname{VFL}}\rVert_2 - \lVert\widehat{\theta}_t^{\operatorname{Central}}\rVert_2$ and $\lVert\widehat{\theta}_t^{\operatorname{Partial}}\rVert_2 - \lVert\widehat{\theta}_t^{\operatorname{Central}}\rVert_2$ to check the reconstruction error in the estimator for each bandit model. Our VFCB metrics quickly recover to the centralized bandits at almost 0 difference, while the partial recovers much slower. Additionally, the cumulative regret difference between partial bandits and full bandits grows continuously to more than 250 while our VFCB protocols difference remains at almost 0. 
Hence, we can show that the performance gap between VFCB and their central variants are small as our VFCB protocols can indeed achieve same cumulative regret target as LinUCB and LinTS.

\vspace{1ex}
\nosection{Extra computation and communication cost for the VFCB}
\begin{figure}[t!]
    \centering
    \subfloat[\label{comp-cost-comparison-vfucb}]{\includegraphics[width=0.4\textwidth]{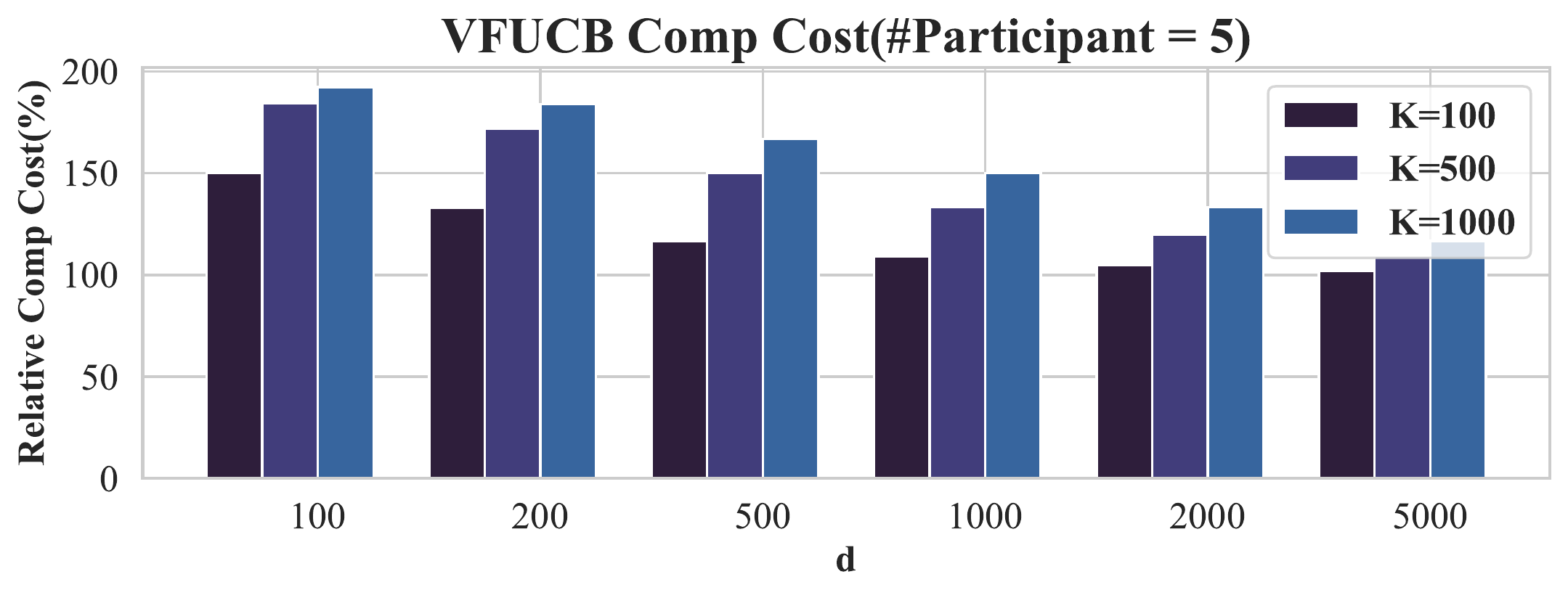}}
    \newline
    \subfloat[\label{comp-cost-comparison-vfts}]{\includegraphics[width=0.4\textwidth]{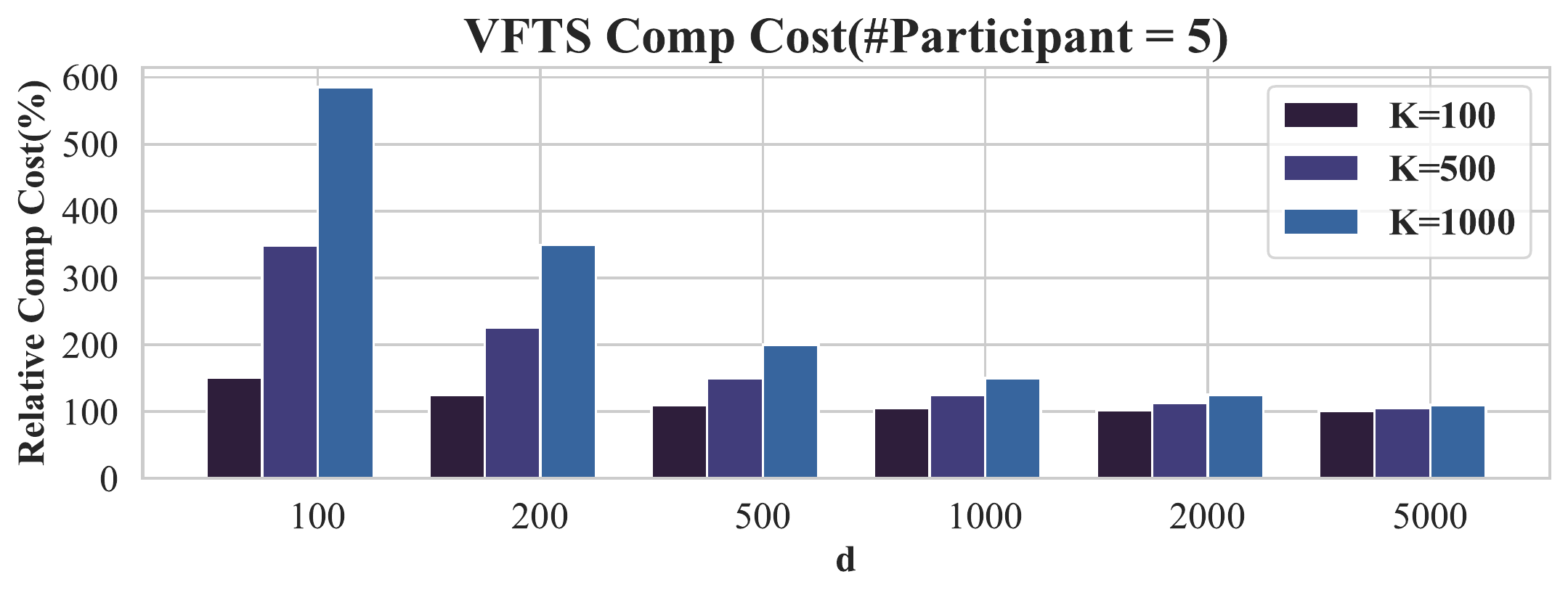}}
    \newline
    \subfloat[\label{comm-cost-comparison-vfl}]{\includegraphics[width=0.4\textwidth]{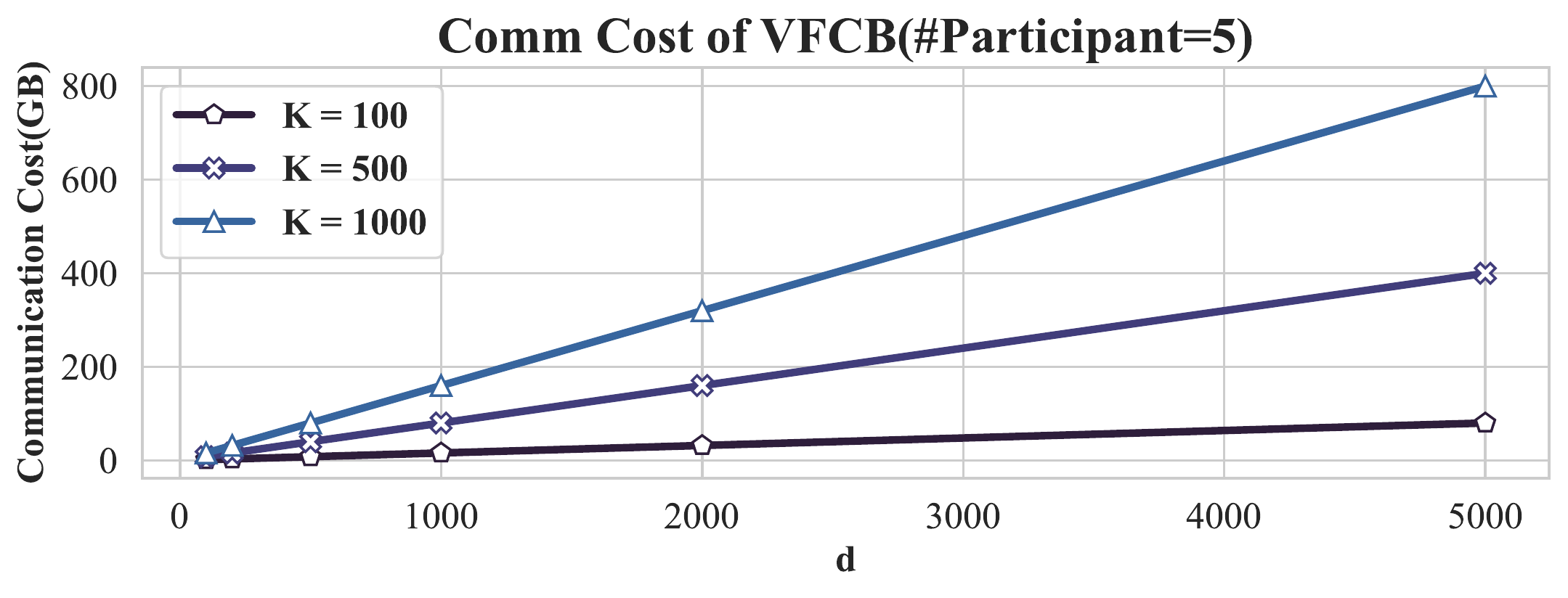}}
    \newline
    \vspace{-2ex}
    \caption{Complexity costs for (a) Relative Computational Cost for VFUCB. (b) Relative Computational Cost for VFTS. (c) Communication Cost in GB for VFCB active participant.}
    \label{fig:Complexity-Comparison}
\end{figure}

\label{nosec:answer-2}
We first perform a synthetic calculation for VFCB protocols in terms of computational cost. For computational cost,we assume that a single arithmetic operation, comparison  for one
element of the matrix has 1 operation cost. For other known cost outlined in prior analysis, we follow them directly in the computation. \footnote{The details for computational complexity analysis and assumptions can be find at Appendix~\ref{appendix:complexity}.} We present the computational cost difference in Figure~(\ref{comp-cost-comparison-vfucb}) and Figure~(\ref{comp-cost-comparison-vfts}) under realistic settings: Steps $T=5000$, item size $K=100,500,1000$ and number of participant $M=5$. For ease of comparison, We use the relative computational cost\footnote{The relative cost can be calculate via $cost_{relative} = \cfrac{cost_{VFCB}}{cost_{central}}$.} in our figure. As shown in Figure~(\ref{comp-cost-comparison-vfucb}), VFUCB only increases computational cost by a maximum of 2 fold on a typical participant setting. Further, this increase is marginal after either item size $K$, or contextual dimension $d$ grows significant. This shows that compared with LinUCB, VFUCB does not introduced a significant computational complexity burden. As of VFTS, due to the O3M complexity introduced as $O(K\times d^2)$, the relative cost is initially higher at above 5x, but decrease to almost identical with LinTS as context dimension $d$ grows larger, resulting higher $O(d^3)$ term. Additionally, due to the expensive cost of sampling a multivariate normal distribution, the effect will be marginal practically. After all, the overall computational cost for our VFCB is still much smaller compared with the high complexity introduced from homomorphic encryption. 

We then perform communication cost analysis in Figure~(\ref{comm-cost-comparison-vfl}). We assume that each array element is a double floating point number of 8 byte and each array is transferred without any compression. This means that our communication cost calculated here should be an upper bound. Since the communication process is contributed by O3M only, both VFCB protocols should have the same communication cost. We present the cost with $T=5000$ under unit in GB at item size $k=100,500$ and $1000$. The result shows that VFCB's communication cost grows linearly in terms of item size $K$ and context dimension $d$. Nontheless, even under extreme setting of $d=5000, K=1000$ the data transfer is still relative small, at 6.25GB per step. Realistic settings would have much smaller contextual size, resulting in smaller communication per step. For example $d=1000,K=1000$ will only have 0.04GB per step. Additionally, our prior analysis shows that the communication complexity also grows linearly with number of participants $M$. However, since the number of participant in a vertical federated setting is much smaller compared with horizontal federated setting, the cost increase is much smaller compared with other two variables. As this is an upper bound estimated without compression, by incorporating other data compression techniques, this cost is comparable with the high communication cost introduced in MPC.


\begin{figure}[t]
    \centering
    \subfloat[\label{partial-data-comparison-synthetic-ucb}]{\includegraphics[width=0.24\textwidth]{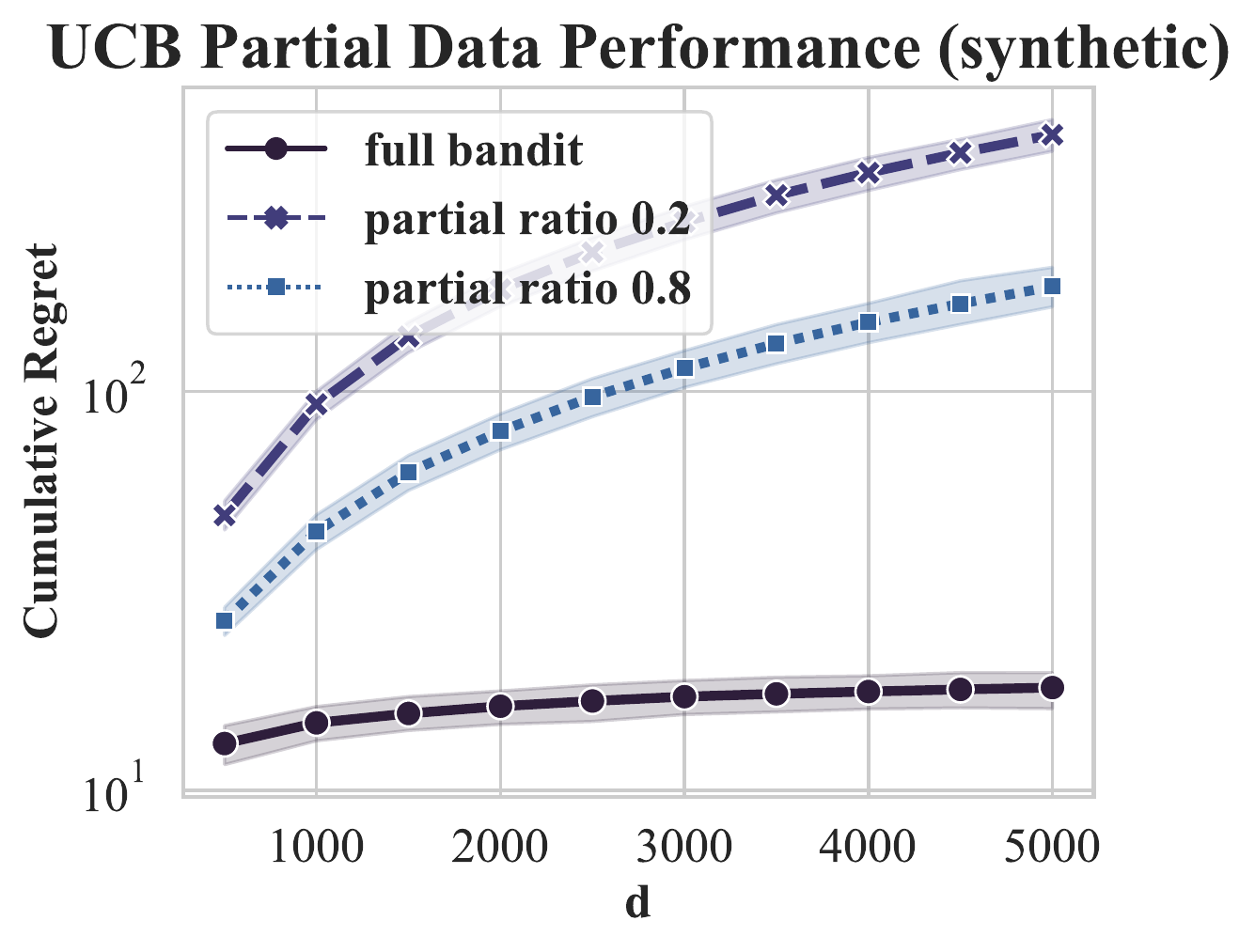}}
    \subfloat[\label{partial-data-comparison-synthetic-ts}]{\includegraphics[width=0.24\textwidth]{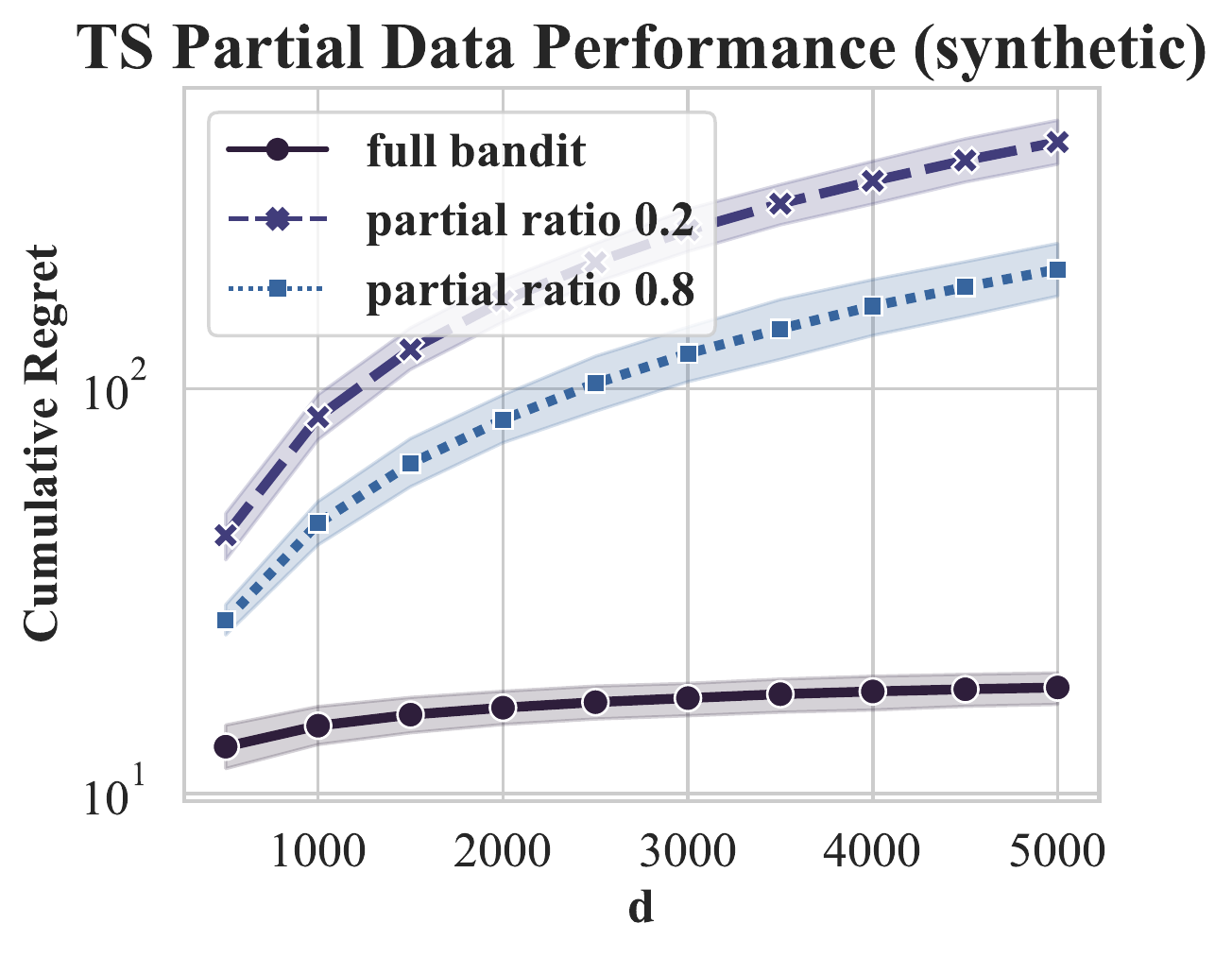}}
    \newline
    \subfloat[\label{partial-data-comparison-criteo-ucb}]{\includegraphics[width=0.45\textwidth]{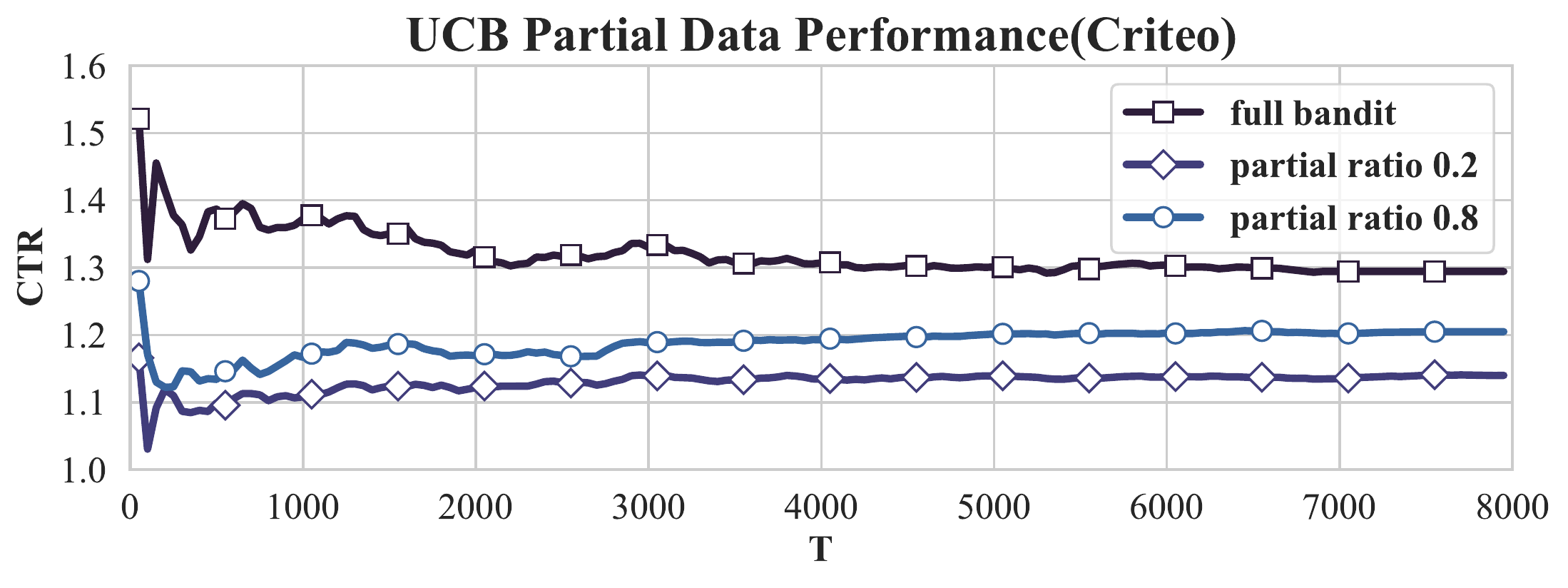}}
    \vspace{-2ex}
    \caption{Experimental results for (a) Regret performance with 20\% and 80\% partial data for LinUCB on synthetic data. (b) Regret performance with 20\% and 80\% partial data for LinTS on synthetic data. (c) Comparison between relative CTR for VFUCB and Partial LinUCB. For ease of comparison, the y scale for (a) and (b) is changed to log scale to differentiate between VFCB and partial bandits.}
    \label{fig:Partial-Comparison}
\end{figure}
\vspace{1ex}
\nosection{The effect of incorporating more data from other participants}
\label{nosec:answer-3}
We demonstrate that from both of our synthetic dataset in Figure~(\ref{partial-data-comparison-synthetic-ucb}) and Figure~(\ref{partial-data-comparison-synthetic-ts}), our contextual bandits' performance increase when we incorporate more data from other participants. The synthetic cumulative regret bound is 10x smaller when using the full data compared with using partial data. Additionally, compared with partial ratio\footnote{The partial ratio is $\cfrac{\#data_{used}}{\#data_{total}}$.}=0.2, contextual bandit at partial ratio=0.8 has smaller cumulative regret. 
To ensure the effect of incorporating all data is applicable on the real-world datasets, we perform experiment on real-world dataset shown in Figure~(\ref{partial-data-comparison-criteo-ucb}). There is an increase of 0.1 in relative CTR metric from partial ratio=0.2 to partial ratio=0.8, and another 0.1 increase to full LinUCB. This shows that the positive effect of incorporating additional data from different department is also true on real-world dataset.
Therefore, using more context data from other department helps improve the performance of contextual bandit algorithm, providing insight to improve the recommendation service quality in a federated manner. %

\nosection{Summary} In summary, our experiments show that our VFCB protocols is indeed lossless compared with the centralized variant. We then conduct synthetic analysis on the extra computational cost and communication cost introduced by O3M on VFCB, and show that the effect is still manageable. Finally, the experiment on both synthetic and real-world dataset suggest that our VFCB protocols can indeed benefit from additional data under the vertical federated settings. Therefore, these 3 answers combined to show the overall value of our VFCB protocols.

\section{Conclusion}
This paper studies an important problem of building online bandit algorithms in the vertical federated setting. This problem is critical for many real-world scenarios where data privacy and compliance are considered. We point out, for this specific problem, we can use a simple yet effective privacy-enhancing technique named O3M to avoid heavy cryptographic operations. Therefore the proposed protocol is very practical and promising for real business products. Since we are the first to formally study this problem to our best knowledge, there still remain many interesting directions to be explored. For example, can we adopt a similar strategy to the neural bandit in the vertical federated setting? We hope our research can provide insights into how to build an efficient bandit algorithm with privacy constraints and draw more attention from the community in this direction.

\bibliographystyle{ACM-Reference-Format}
\bibliography{reference}


\begin{thebibliography}{35}


\ifx \showCODEN    \undefined \def \showCODEN     #1{\unskip}     \fi
\ifx \showDOI      \undefined \def \showDOI       #1{#1}\fi
\ifx \showISBNx    \undefined \def \showISBNx     #1{\unskip}     \fi
\ifx \showISBNxiii \undefined \def \showISBNxiii  #1{\unskip}     \fi
\ifx \showISSN     \undefined \def \showISSN      #1{\unskip}     \fi
\ifx \showLCCN     \undefined \def \showLCCN      #1{\unskip}     \fi
\ifx \shownote     \undefined \def \shownote      #1{#1}          \fi
\ifx \showarticletitle \undefined \def \showarticletitle #1{#1}   \fi
\ifx \showURL      \undefined \def \showURL       {\relax}        \fi
\providecommand\bibfield[2]{#2}
\providecommand\bibinfo[2]{#2}
\providecommand\natexlab[1]{#1}
\providecommand\showeprint[2][]{arXiv:#2}

\bibitem[\protect\citeauthoryear{Abbasi{-}Yadkori, P{\'{a}}l, and
  Szepesv{\'{a}}ri}{Abbasi{-}Yadkori et~al\mbox{.}}{2011}]%
        {DBLP:conf/nips/Abbasi-YadkoriPS11}
\bibfield{author}{\bibinfo{person}{Yasin Abbasi{-}Yadkori},
  \bibinfo{person}{D{\'{a}}vid P{\'{a}}l}, {and} \bibinfo{person}{Csaba
  Szepesv{\'{a}}ri}.} \bibinfo{year}{2011}\natexlab{}.
\newblock \showarticletitle{Improved Algorithms for Linear Stochastic Bandits}.
  In \bibinfo{booktitle}{\emph{Advances in Neural Information Processing
  Systems 24: 25th Annual Conference on Neural Information Processing Systems
  2011. Proceedings of a meeting held 12-14 December 2011, Granada, Spain}},
  \bibfield{editor}{\bibinfo{person}{John Shawe{-}Taylor},
  \bibinfo{person}{Richard~S. Zemel}, \bibinfo{person}{Peter~L. Bartlett},
  \bibinfo{person}{Fernando C.~N. Pereira}, {and} \bibinfo{person}{Kilian~Q.
  Weinberger}} (Eds.). \bibinfo{pages}{2312--2320}.
\newblock
\urldef\tempurl%
\url{https://proceedings.neurips.cc/paper/2011/hash/e1d5be1c7f2f456670de3d53c7b54f4a-Abstract.html}
\showURL{%
\tempurl}


\bibitem[\protect\citeauthoryear{Agarwal, Langford, and Wei}{Agarwal
  et~al\mbox{.}}{2020}]%
        {agarwal2020federated}
\bibfield{author}{\bibinfo{person}{Alekh Agarwal}, \bibinfo{person}{John
  Langford}, {and} \bibinfo{person}{Chen-Yu Wei}.}
  \bibinfo{year}{2020}\natexlab{}.
\newblock \showarticletitle{Federated residual learning}.
\newblock \bibinfo{journal}{\emph{arXiv preprint arXiv:2003.12880}}
  (\bibinfo{year}{2020}).
\newblock


\bibitem[\protect\citeauthoryear{Agrawal and Goyal}{Agrawal and Goyal}{2013}]%
        {agrawal2013thompson}
\bibfield{author}{\bibinfo{person}{Shipra Agrawal} {and} \bibinfo{person}{Navin
  Goyal}.} \bibinfo{year}{2013}\natexlab{}.
\newblock \showarticletitle{Thompson sampling for contextual bandits with
  linear payoffs}. In \bibinfo{booktitle}{\emph{International conference on
  machine learning}}. PMLR, \bibinfo{pages}{127--135}.
\newblock


\bibitem[\protect\citeauthoryear{Chai, Wang, Zhang, Yang, Cai, Chen, and
  Yang}{Chai et~al\mbox{.}}{2022}]%
        {DBLP:conf/kdd/ChaiWZYC0022}
\bibfield{author}{\bibinfo{person}{Di Chai}, \bibinfo{person}{Leye Wang},
  \bibinfo{person}{Junxue Zhang}, \bibinfo{person}{Liu Yang},
  \bibinfo{person}{Shuowei Cai}, \bibinfo{person}{Kai Chen}, {and}
  \bibinfo{person}{Qiang Yang}.} \bibinfo{year}{2022}\natexlab{}.
\newblock \showarticletitle{Practical Lossless Federated Singular Vector
  Decomposition over Billion-Scale Data}. In \bibinfo{booktitle}{\emph{{KDD}
  '22: The 28th {ACM} {SIGKDD} Conference on Knowledge Discovery and Data
  Mining, Washington, DC, USA, August 14 - 18, 2022}},
  \bibfield{editor}{\bibinfo{person}{Aidong Zhang} {and}
  \bibinfo{person}{Huzefa Rangwala}} (Eds.). \bibinfo{publisher}{{ACM}},
  \bibinfo{pages}{46--55}.
\newblock
\urldef\tempurl%
\url{https://doi.org/10.1145/3534678.3539402}
\showDOI{\tempurl}


\bibitem[\protect\citeauthoryear{Chen, Jin, Sun, and Yin}{Chen
  et~al\mbox{.}}{2020}]%
        {chen2020vafl}
\bibfield{author}{\bibinfo{person}{Tianyi Chen}, \bibinfo{person}{Xiao Jin},
  \bibinfo{person}{Yuejiao Sun}, {and} \bibinfo{person}{Wotao Yin}.}
  \bibinfo{year}{2020}\natexlab{}.
\newblock \showarticletitle{Vafl: a method of vertical asynchronous federated
  learning}.
\newblock \bibinfo{journal}{\emph{arXiv preprint arXiv:2007.06081}}
  (\bibinfo{year}{2020}).
\newblock


\bibitem[\protect\citeauthoryear{Cheung, Jiang, Yu, and Lou}{Cheung
  et~al\mbox{.}}{2022}]%
        {DBLP:journals/corr/abs-2203-01752}
\bibfield{author}{\bibinfo{person}{Yiu{-}ming Cheung}, \bibinfo{person}{Juyong
  Jiang}, \bibinfo{person}{Feng Yu}, {and} \bibinfo{person}{Jian Lou}.}
  \bibinfo{year}{2022}\natexlab{}.
\newblock \showarticletitle{Vertical Federated Principal Component Analysis and
  Its Kernel Extension on Feature-wise Distributed Data}.
\newblock \bibinfo{journal}{\emph{CoRR}}  \bibinfo{volume}{abs/2203.01752}
  (\bibinfo{year}{2022}).
\newblock
\urldef\tempurl%
\url{https://doi.org/10.48550/arXiv.2203.01752}
\showDOI{\tempurl}
\showeprint[arXiv]{2203.01752}


\bibitem[\protect\citeauthoryear{Dubey and Pentland}{Dubey and
  Pentland}{2020}]%
        {DBLP:conf/nips/DubeyP20}
\bibfield{author}{\bibinfo{person}{Abhimanyu Dubey} {and}
  \bibinfo{person}{Alex~'Sandy' Pentland}.} \bibinfo{year}{2020}\natexlab{}.
\newblock \showarticletitle{Differentially-Private Federated Linear Bandits}.
  In \bibinfo{booktitle}{\emph{Advances in Neural Information Processing
  Systems 33: Annual Conference on Neural Information Processing Systems 2020,
  NeurIPS 2020, December 6-12, 2020, virtual}},
  \bibfield{editor}{\bibinfo{person}{Hugo Larochelle},
  \bibinfo{person}{Marc'Aurelio Ranzato}, \bibinfo{person}{Raia Hadsell},
  \bibinfo{person}{Maria{-}Florina Balcan}, {and} \bibinfo{person}{Hsuan{-}Tien
  Lin}} (Eds.).
\newblock
\urldef\tempurl%
\url{https://proceedings.neurips.cc/paper/2020/hash/4311359ed4969e8401880e3c1836fbe1-Abstract.html}
\showURL{%
\tempurl}


\bibitem[\protect\citeauthoryear{Dwork}{Dwork}{2008}]%
        {dwork2008differential}
\bibfield{author}{\bibinfo{person}{Cynthia Dwork}.}
  \bibinfo{year}{2008}\natexlab{}.
\newblock \showarticletitle{Differential privacy: A survey of results}. In
  \bibinfo{booktitle}{\emph{International conference on theory and applications
  of models of computation}}. Springer, \bibinfo{pages}{1--19}.
\newblock


\bibitem[\protect\citeauthoryear{Dwork, Roth, et~al\mbox{.}}{Dwork
  et~al\mbox{.}}{2014}]%
        {dwork2014algorithmic}
\bibfield{author}{\bibinfo{person}{Cynthia Dwork}, \bibinfo{person}{Aaron
  Roth}, {et~al\mbox{.}}} \bibinfo{year}{2014}\natexlab{}.
\newblock \showarticletitle{The algorithmic foundations of differential
  privacy}.
\newblock \bibinfo{journal}{\emph{Foundations and Trends{\textregistered} in
  Theoretical Computer Science}} \bibinfo{volume}{9}, \bibinfo{number}{3--4}
  (\bibinfo{year}{2014}), \bibinfo{pages}{211--407}.
\newblock


\bibitem[\protect\citeauthoryear{Goldreich}{Goldreich}{1998}]%
        {goldreich1998secure}
\bibfield{author}{\bibinfo{person}{Oded Goldreich}.}
  \bibinfo{year}{1998}\natexlab{}.
\newblock \showarticletitle{Secure multi-party computation}.
\newblock \bibinfo{journal}{\emph{Manuscript. Preliminary version}}
  \bibinfo{volume}{78} (\bibinfo{year}{1998}), \bibinfo{pages}{110}.
\newblock


\bibitem[\protect\citeauthoryear{Han, Liang, Wang, and Zhang}{Han
  et~al\mbox{.}}{2021}]%
        {NEURIPS2021_df0e09d6}
\bibfield{author}{\bibinfo{person}{Yuxuan Han}, \bibinfo{person}{Zhipeng
  Liang}, \bibinfo{person}{Yang Wang}, {and} \bibinfo{person}{Jiheng Zhang}.}
  \bibinfo{year}{2021}\natexlab{}.
\newblock \showarticletitle{Generalized Linear Bandits with Local Differential
  Privacy}. In \bibinfo{booktitle}{\emph{Advances in Neural Information
  Processing Systems}}, \bibfield{editor}{\bibinfo{person}{M.~Ranzato},
  \bibinfo{person}{A.~Beygelzimer}, \bibinfo{person}{Y.~Dauphin},
  \bibinfo{person}{P.S. Liang}, {and} \bibinfo{person}{J.~Wortman Vaughan}}
  (Eds.), Vol.~\bibinfo{volume}{34}. \bibinfo{publisher}{Curran Associates,
  Inc.}, \bibinfo{pages}{26511--26522}.
\newblock
\urldef\tempurl%
\url{https://proceedings.neurips.cc/paper/2021/file/df0e09d6f25a15a815563df9827f48fa-Paper.pdf}
\showURL{%
\tempurl}


\bibitem[\protect\citeauthoryear{Hannun, Knott, Sengupta, and van~der
  Maaten}{Hannun et~al\mbox{.}}{2019}]%
        {DBLP:journals/corr/abs-1910-05299}
\bibfield{author}{\bibinfo{person}{Awni~Y. Hannun}, \bibinfo{person}{Brian
  Knott}, \bibinfo{person}{Shubho Sengupta}, {and} \bibinfo{person}{Laurens
  van~der Maaten}.} \bibinfo{year}{2019}\natexlab{}.
\newblock \showarticletitle{Privacy-Preserving Contextual Bandits}.
\newblock \bibinfo{journal}{\emph{CoRR}}  \bibinfo{volume}{abs/1910.05299}
  (\bibinfo{year}{2019}).
\newblock
\showeprint[arXiv]{1910.05299}
\urldef\tempurl%
\url{http://arxiv.org/abs/1910.05299}
\showURL{%
\tempurl}


\bibitem[\protect\citeauthoryear{Hardy, Henecka, Ivey-Law, Nock, Patrini,
  Smith, and Thorne}{Hardy et~al\mbox{.}}{2017}]%
        {hardy2017private}
\bibfield{author}{\bibinfo{person}{Stephen Hardy}, \bibinfo{person}{Wilko
  Henecka}, \bibinfo{person}{Hamish Ivey-Law}, \bibinfo{person}{Richard Nock},
  \bibinfo{person}{Giorgio Patrini}, \bibinfo{person}{Guillaume Smith}, {and}
  \bibinfo{person}{Brian Thorne}.} \bibinfo{year}{2017}\natexlab{}.
\newblock \showarticletitle{Private federated learning on vertically
  partitioned data via entity resolution and additively homomorphic
  encryption}.
\newblock \bibinfo{journal}{\emph{arXiv preprint arXiv:1711.10677}}
  (\bibinfo{year}{2017}).
\newblock


\bibitem[\protect\citeauthoryear{He, Zhang, Kan, and Chua}{He
  et~al\mbox{.}}{2017}]%
        {DBLP:journals/corr/abs-1708-05024}
\bibfield{author}{\bibinfo{person}{Xiangnan He}, \bibinfo{person}{Hanwang
  Zhang}, \bibinfo{person}{Min{-}Yen Kan}, {and} \bibinfo{person}{Tat{-}Seng
  Chua}.} \bibinfo{year}{2017}\natexlab{}.
\newblock \showarticletitle{Fast Matrix Factorization for Online Recommendation
  with Implicit Feedback}.
\newblock \bibinfo{journal}{\emph{CoRR}}  \bibinfo{volume}{abs/1708.05024}
  (\bibinfo{year}{2017}).
\newblock
\showeprint[arXiv]{1708.05024}
\urldef\tempurl%
\url{http://arxiv.org/abs/1708.05024}
\showURL{%
\tempurl}


\bibitem[\protect\citeauthoryear{Hu, Niu, Yang, and Zhou}{Hu
  et~al\mbox{.}}{2019}]%
        {DBLP:conf/kdd/HuNYZ19}
\bibfield{author}{\bibinfo{person}{Yaochen Hu}, \bibinfo{person}{Di Niu},
  \bibinfo{person}{Jianming Yang}, {and} \bibinfo{person}{Shengping Zhou}.}
  \bibinfo{year}{2019}\natexlab{}.
\newblock \showarticletitle{{FDML:} {A} Collaborative Machine Learning
  Framework for Distributed Features}. In \bibinfo{booktitle}{\emph{Proceedings
  of the 25th {ACM} {SIGKDD} International Conference on Knowledge Discovery
  {\&} Data Mining, {KDD} 2019, Anchorage, AK, USA, August 4-8, 2019}},
  \bibfield{editor}{\bibinfo{person}{Ankur Teredesai}, \bibinfo{person}{Vipin
  Kumar}, \bibinfo{person}{Ying Li}, \bibinfo{person}{R{\'{o}}mer Rosales},
  \bibinfo{person}{Evimaria Terzi}, {and} \bibinfo{person}{George Karypis}}
  (Eds.). \bibinfo{publisher}{{ACM}}, \bibinfo{pages}{2232--2240}.
\newblock
\urldef\tempurl%
\url{https://doi.org/10.1145/3292500.3330765}
\showDOI{\tempurl}


\bibitem[\protect\citeauthoryear{Huang, Wu, Yang, and Shen}{Huang
  et~al\mbox{.}}{2021}]%
        {DBLP:conf/nips/HuangWYS21}
\bibfield{author}{\bibinfo{person}{Ruiquan Huang}, \bibinfo{person}{Weiqiang
  Wu}, \bibinfo{person}{Jing Yang}, {and} \bibinfo{person}{Cong Shen}.}
  \bibinfo{year}{2021}\natexlab{}.
\newblock \showarticletitle{Federated Linear Contextual Bandits}. In
  \bibinfo{booktitle}{\emph{Advances in Neural Information Processing Systems
  34: Annual Conference on Neural Information Processing Systems 2021, NeurIPS
  2021, December 6-14, 2021, virtual}},
  \bibfield{editor}{\bibinfo{person}{Marc'Aurelio Ranzato},
  \bibinfo{person}{Alina Beygelzimer}, \bibinfo{person}{Yann~N. Dauphin},
  \bibinfo{person}{Percy Liang}, {and} \bibinfo{person}{Jennifer~Wortman
  Vaughan}} (Eds.). \bibinfo{pages}{27057--27068}.
\newblock
\urldef\tempurl%
\url{https://proceedings.neurips.cc/paper/2021/hash/e347c51419ffb23ca3fd5050202f9c3d-Abstract.html}
\showURL{%
\tempurl}


\bibitem[\protect\citeauthoryear{Labs}{Labs}{2014}]%
        {criteo_ai_labs_2014}
\bibfield{author}{\bibinfo{person}{Criteo~AI Labs}.}
  \bibinfo{year}{2014}\natexlab{}.
\newblock
\newblock
\urldef\tempurl%
\url{https://labs.criteo.com/2014/02/kaggle-display-advertising-challenge-dataset/}
\showURL{%
\tempurl}


\bibitem[\protect\citeauthoryear{Lattimore and Szepesv{\'a}ri}{Lattimore and
  Szepesv{\'a}ri}{2020}]%
        {lattimore2020bandit}
\bibfield{author}{\bibinfo{person}{Tor Lattimore} {and} \bibinfo{person}{Csaba
  Szepesv{\'a}ri}.} \bibinfo{year}{2020}\natexlab{}.
\newblock \bibinfo{booktitle}{\emph{Bandit algorithms}}.
\newblock \bibinfo{publisher}{Cambridge University Press}.
\newblock


\bibitem[\protect\citeauthoryear{Li and Wang}{Li and Wang}{2022}]%
        {DBLP:conf/aistats/LiW22}
\bibfield{author}{\bibinfo{person}{Chuanhao Li} {and} \bibinfo{person}{Hongning
  Wang}.} \bibinfo{year}{2022}\natexlab{}.
\newblock \showarticletitle{Asynchronous Upper Confidence Bound Algorithms for
  Federated Linear Bandits}. In \bibinfo{booktitle}{\emph{International
  Conference on Artificial Intelligence and Statistics, {AISTATS} 2022, 28-30
  March 2022, Virtual Event}} \emph{(\bibinfo{series}{Proceedings of Machine
  Learning Research}, Vol.~\bibinfo{volume}{151})},
  \bibfield{editor}{\bibinfo{person}{Gustau Camps{-}Valls},
  \bibinfo{person}{Francisco J.~R. Ruiz}, {and} \bibinfo{person}{Isabel
  Valera}} (Eds.). \bibinfo{publisher}{{PMLR}}, \bibinfo{pages}{6529--6553}.
\newblock
\urldef\tempurl%
\url{https://proceedings.mlr.press/v151/li22e.html}
\showURL{%
\tempurl}


\bibitem[\protect\citeauthoryear{Li, Chu, Langford, and Schapire}{Li
  et~al\mbox{.}}{2010}]%
        {DBLP:conf/www/LiCLS10}
\bibfield{author}{\bibinfo{person}{Lihong Li}, \bibinfo{person}{Wei Chu},
  \bibinfo{person}{John Langford}, {and} \bibinfo{person}{Robert~E. Schapire}.}
  \bibinfo{year}{2010}\natexlab{}.
\newblock \showarticletitle{A contextual-bandit approach to personalized news
  article recommendation}. In \bibinfo{booktitle}{\emph{Proceedings of the 19th
  International Conference on World Wide Web, {WWW} 2010, Raleigh, North
  Carolina, USA, April 26-30, 2010}},
  \bibfield{editor}{\bibinfo{person}{Michael Rappa}, \bibinfo{person}{Paul
  Jones}, \bibinfo{person}{Juliana Freire}, {and} \bibinfo{person}{Soumen
  Chakrabarti}} (Eds.). \bibinfo{publisher}{{ACM}}, \bibinfo{pages}{661--670}.
\newblock
\urldef\tempurl%
\url{https://doi.org/10.1145/1772690.1772758}
\showDOI{\tempurl}


\bibitem[\protect\citeauthoryear{Li, Song, and Fragouli}{Li
  et~al\mbox{.}}{2020}]%
        {li2020federated}
\bibfield{author}{\bibinfo{person}{Tan Li}, \bibinfo{person}{Linqi Song}, {and}
  \bibinfo{person}{Christina Fragouli}.} \bibinfo{year}{2020}\natexlab{}.
\newblock \showarticletitle{Federated recommendation system via differential
  privacy}. In \bibinfo{booktitle}{\emph{2020 IEEE International Symposium on
  Information Theory (ISIT)}}. IEEE, \bibinfo{pages}{2592--2597}.
\newblock


\bibitem[\protect\citeauthoryear{Lika, Kolomvatsos, and Hadjiefthymiades}{Lika
  et~al\mbox{.}}{2014}]%
        {DBLP:journals/eswa/LikaKH14}
\bibfield{author}{\bibinfo{person}{Blerina Lika}, \bibinfo{person}{Kostas
  Kolomvatsos}, {and} \bibinfo{person}{Stathes Hadjiefthymiades}.}
  \bibinfo{year}{2014}\natexlab{}.
\newblock \showarticletitle{Facing the cold start problem in recommender
  systems}.
\newblock \bibinfo{journal}{\emph{Expert Syst. Appl.}} \bibinfo{volume}{41},
  \bibinfo{number}{4} (\bibinfo{year}{2014}), \bibinfo{pages}{2065--2073}.
\newblock
\urldef\tempurl%
\url{https://doi.org/10.1016/j.eswa.2013.09.005}
\showDOI{\tempurl}


\bibitem[\protect\citeauthoryear{Malekzadeh, Athanasakis, Haddadi, and
  Livshits}{Malekzadeh et~al\mbox{.}}{2020}]%
        {DBLP:conf/mlsys/MalekzadehAHL20}
\bibfield{author}{\bibinfo{person}{Mohammad Malekzadeh},
  \bibinfo{person}{Dimitrios Athanasakis}, \bibinfo{person}{Hamed Haddadi},
  {and} \bibinfo{person}{Benjamin Livshits}.} \bibinfo{year}{2020}\natexlab{}.
\newblock \showarticletitle{Privacy-Preserving Bandits}. In
  \bibinfo{booktitle}{\emph{Proceedings of Machine Learning and Systems 2020,
  MLSys 2020, Austin, TX, USA, March 2-4, 2020}},
  \bibfield{editor}{\bibinfo{person}{Inderjit~S. Dhillon},
  \bibinfo{person}{Dimitris~S. Papailiopoulos}, {and} \bibinfo{person}{Vivienne
  Sze}} (Eds.). \bibinfo{publisher}{mlsys.org}.
\newblock
\urldef\tempurl%
\url{https://proceedings.mlsys.org/book/310.pdf}
\showURL{%
\tempurl}


\bibitem[\protect\citeauthoryear{Ren, Yang, and Chen}{Ren
  et~al\mbox{.}}{2022}]%
        {DBLP:journals/tist/RenYC22}
\bibfield{author}{\bibinfo{person}{Zhenghang Ren}, \bibinfo{person}{Liu Yang},
  {and} \bibinfo{person}{Kai Chen}.} \bibinfo{year}{2022}\natexlab{}.
\newblock \showarticletitle{Improving Availability of Vertical Federated
  Learning: Relaxing Inference on Non-overlapping Data}.
\newblock \bibinfo{journal}{\emph{{ACM} Trans. Intell. Syst. Technol.}}
  \bibinfo{volume}{13}, \bibinfo{number}{4} (\bibinfo{year}{2022}),
  \bibinfo{pages}{58:1--58:20}.
\newblock
\urldef\tempurl%
\url{https://doi.org/10.1145/3501817}
\showDOI{\tempurl}


\bibitem[\protect\citeauthoryear{Romanini, Hall, Papadopoulos, Titcombe,
  Ismail, Cebere, Sandmann, Roehm, and Hoeh}{Romanini et~al\mbox{.}}{2021}]%
        {DBLP:journals/corr/abs-2104-00489}
\bibfield{author}{\bibinfo{person}{Daniele Romanini},
  \bibinfo{person}{Adam~James Hall}, \bibinfo{person}{Pavlos Papadopoulos},
  \bibinfo{person}{Tom Titcombe}, \bibinfo{person}{Abbas Ismail},
  \bibinfo{person}{Tudor Cebere}, \bibinfo{person}{Robert Sandmann},
  \bibinfo{person}{Robin Roehm}, {and} \bibinfo{person}{Michael~A. Hoeh}.}
  \bibinfo{year}{2021}\natexlab{}.
\newblock \showarticletitle{PyVertical: {A} Vertical Federated Learning
  Framework for Multi-headed SplitNN}.
\newblock \bibinfo{journal}{\emph{CoRR}}  \bibinfo{volume}{abs/2104.00489}
  (\bibinfo{year}{2021}).
\newblock
\showeprint[arXiv]{2104.00489}
\urldef\tempurl%
\url{https://arxiv.org/abs/2104.00489}
\showURL{%
\tempurl}


\bibitem[\protect\citeauthoryear{Shariff and Sheffet}{Shariff and
  Sheffet}{2018}]%
        {DBLP:conf/nips/ShariffS18}
\bibfield{author}{\bibinfo{person}{Roshan Shariff} {and} \bibinfo{person}{Or
  Sheffet}.} \bibinfo{year}{2018}\natexlab{}.
\newblock \showarticletitle{Differentially Private Contextual Linear Bandits}.
  In \bibinfo{booktitle}{\emph{Advances in Neural Information Processing
  Systems 31: Annual Conference on Neural Information Processing Systems 2018,
  NeurIPS 2018, December 3-8, 2018, Montr{\'{e}}al, Canada}},
  \bibfield{editor}{\bibinfo{person}{Samy Bengio}, \bibinfo{person}{Hanna~M.
  Wallach}, \bibinfo{person}{Hugo Larochelle}, \bibinfo{person}{Kristen
  Grauman}, \bibinfo{person}{Nicol{\`{o}} Cesa{-}Bianchi}, {and}
  \bibinfo{person}{Roman Garnett}} (Eds.). \bibinfo{pages}{4301--4311}.
\newblock
\urldef\tempurl%
\url{https://proceedings.neurips.cc/paper/2018/hash/a1d7311f2a312426d710e1c617fcbc8c-Abstract.html}
\showURL{%
\tempurl}


\bibitem[\protect\citeauthoryear{Shi, Shen, and Yang}{Shi
  et~al\mbox{.}}{2021a}]%
        {shi2021federated}
\bibfield{author}{\bibinfo{person}{Chengshuai Shi}, \bibinfo{person}{Cong
  Shen}, {and} \bibinfo{person}{Jing Yang}.} \bibinfo{year}{2021}\natexlab{a}.
\newblock \showarticletitle{Federated multi-armed bandits with
  personalization}. In \bibinfo{booktitle}{\emph{International Conference on
  Artificial Intelligence and Statistics}}. PMLR, \bibinfo{pages}{2917--2925}.
\newblock


\bibitem[\protect\citeauthoryear{Shi, Shen, and Yang}{Shi
  et~al\mbox{.}}{2021b}]%
        {DBLP:conf/aistats/ShiSY21}
\bibfield{author}{\bibinfo{person}{Chengshuai Shi}, \bibinfo{person}{Cong
  Shen}, {and} \bibinfo{person}{Jing Yang}.} \bibinfo{year}{2021}\natexlab{b}.
\newblock \showarticletitle{Federated Multi-armed Bandits with
  Personalization}. In \bibinfo{booktitle}{\emph{The 24th International
  Conference on Artificial Intelligence and Statistics, {AISTATS} 2021, April
  13-15, 2021, Virtual Event}} \emph{(\bibinfo{series}{Proceedings of Machine
  Learning Research}, Vol.~\bibinfo{volume}{130})},
  \bibfield{editor}{\bibinfo{person}{Arindam Banerjee} {and}
  \bibinfo{person}{Kenji Fukumizu}} (Eds.). \bibinfo{publisher}{{PMLR}},
  \bibinfo{pages}{2917--2925}.
\newblock
\urldef\tempurl%
\url{http://proceedings.mlr.press/v130/shi21c.html}
\showURL{%
\tempurl}


\bibitem[\protect\citeauthoryear{Slivkins et~al\mbox{.}}{Slivkins
  et~al\mbox{.}}{2019}]%
        {slivkins2019introduction}
\bibfield{author}{\bibinfo{person}{Aleksandrs Slivkins} {et~al\mbox{.}}}
  \bibinfo{year}{2019}\natexlab{}.
\newblock \showarticletitle{Introduction to multi-armed bandits}.
\newblock \bibinfo{journal}{\emph{Foundations and Trends{\textregistered} in
  Machine Learning}} \bibinfo{volume}{12}, \bibinfo{number}{1-2}
  (\bibinfo{year}{2019}), \bibinfo{pages}{1--286}.
\newblock


\bibitem[\protect\citeauthoryear{Tewari and Murphy}{Tewari and Murphy}{2017}]%
        {DBLP:books/sp/17/TewariM17}
\bibfield{author}{\bibinfo{person}{Ambuj Tewari} {and}
  \bibinfo{person}{Susan~A. Murphy}.} \bibinfo{year}{2017}\natexlab{}.
\newblock \showarticletitle{From Ads to Interventions: Contextual Bandits in
  Mobile Health}.
\newblock In \bibinfo{booktitle}{\emph{Mobile Health - Sensors, Analytic
  Methods, and Applications}}, \bibfield{editor}{\bibinfo{person}{James~M.
  Rehg}, \bibinfo{person}{Susan~A. Murphy}, {and} \bibinfo{person}{Santosh
  Kumar}} (Eds.). \bibinfo{publisher}{Springer}, \bibinfo{pages}{495--517}.
\newblock
\urldef\tempurl%
\url{https://doi.org/10.1007/978-3-319-51394-2\_25}
\showDOI{\tempurl}


\bibitem[\protect\citeauthoryear{Wainwright}{Wainwright}{2019}]%
        {wainwright2019high}
\bibfield{author}{\bibinfo{person}{Martin~J Wainwright}.}
  \bibinfo{year}{2019}\natexlab{}.
\newblock \bibinfo{booktitle}{\emph{High-dimensional statistics: A
  non-asymptotic viewpoint}}. Vol.~\bibinfo{volume}{48}.
\newblock \bibinfo{publisher}{Cambridge University Press}.
\newblock


\bibitem[\protect\citeauthoryear{Wei, Li, Ma, Ding, Wei, Wu, Chen, and
  Ranbaduge}{Wei et~al\mbox{.}}{2022}]%
        {DBLP:journals/corr/abs-2202-04309}
\bibfield{author}{\bibinfo{person}{Kang Wei}, \bibinfo{person}{Jun Li},
  \bibinfo{person}{Chuan Ma}, \bibinfo{person}{Ming Ding}, \bibinfo{person}{Sha
  Wei}, \bibinfo{person}{Fan Wu}, \bibinfo{person}{Guihai Chen}, {and}
  \bibinfo{person}{Thilina Ranbaduge}.} \bibinfo{year}{2022}\natexlab{}.
\newblock \showarticletitle{Vertical Federated Learning: Challenges,
  Methodologies and Experiments}.
\newblock \bibinfo{journal}{\emph{CoRR}}  \bibinfo{volume}{abs/2202.04309}
  (\bibinfo{year}{2022}).
\newblock
\showeprint[arXiv]{2202.04309}
\urldef\tempurl%
\url{https://arxiv.org/abs/2202.04309}
\showURL{%
\tempurl}


\bibitem[\protect\citeauthoryear{Wu, Cai, Xiao, Chen, and Ooi}{Wu
  et~al\mbox{.}}{2020}]%
        {wu2020privacy}
\bibfield{author}{\bibinfo{person}{Yuncheng Wu}, \bibinfo{person}{Shaofeng
  Cai}, \bibinfo{person}{Xiaokui Xiao}, \bibinfo{person}{Gang Chen}, {and}
  \bibinfo{person}{Beng~Chin Ooi}.} \bibinfo{year}{2020}\natexlab{}.
\newblock \showarticletitle{Privacy preserving vertical federated learning for
  tree-based models}.
\newblock \bibinfo{journal}{\emph{arXiv preprint arXiv:2008.06170}}
  (\bibinfo{year}{2020}).
\newblock


\bibitem[\protect\citeauthoryear{Zheng, Cai, Huang, Li, and Wang}{Zheng
  et~al\mbox{.}}{2020}]%
        {DBLP:conf/nips/0007C0L020}
\bibfield{author}{\bibinfo{person}{Kai Zheng}, \bibinfo{person}{Tianle Cai},
  \bibinfo{person}{Weiran Huang}, \bibinfo{person}{Zhenguo Li}, {and}
  \bibinfo{person}{Liwei Wang}.} \bibinfo{year}{2020}\natexlab{}.
\newblock \showarticletitle{Locally Differentially Private (Contextual) Bandits
  Learning}. In \bibinfo{booktitle}{\emph{Advances in Neural Information
  Processing Systems 33: Annual Conference on Neural Information Processing
  Systems 2020, NeurIPS 2020, December 6-12, 2020, virtual}},
  \bibfield{editor}{\bibinfo{person}{Hugo Larochelle},
  \bibinfo{person}{Marc'Aurelio Ranzato}, \bibinfo{person}{Raia Hadsell},
  \bibinfo{person}{Maria{-}Florina Balcan}, {and} \bibinfo{person}{Hsuan{-}Tien
  Lin}} (Eds.).
\newblock
\urldef\tempurl%
\url{https://proceedings.neurips.cc/paper/2020/hash/908c9a564a86426585b29f5335b619bc-Abstract.html}
\showURL{%
\tempurl}


\bibitem[\protect\citeauthoryear{Zhu, Zhu, Liu, and Liu}{Zhu
  et~al\mbox{.}}{2021}]%
        {zhu2021federated}
\bibfield{author}{\bibinfo{person}{Zhaowei Zhu}, \bibinfo{person}{Jingxuan
  Zhu}, \bibinfo{person}{Ji Liu}, {and} \bibinfo{person}{Yang Liu}.}
  \bibinfo{year}{2021}\natexlab{}.
\newblock \showarticletitle{Federated bandit: A gossiping approach}. In
  \bibinfo{booktitle}{\emph{Abstract Proceedings of the 2021 ACM
  SIGMETRICS/International Conference on Measurement and Modeling of Computer
  Systems}}. \bibinfo{pages}{3--4}.
\newblock


\end{thebibliography}

\appendix
\section{VFTS}
\subsection{VFTS Algorithm}
\label{appendix:VFTS}
\begin{algorithm}[ht]
 \caption{VFTS}
 \label{alg:VFTS}
\begin{algorithmic}[1]
  \STATE \textbf{Input:} prior variance $v$.
  \STATE Initial $\tilde{\Lambda}_0 = \mathbf{I}_d$, $\tilde{u}_0 = \mathbf{0}$, and $\hat{\theta}_1 =  \mathbf{0}$
  \STATE Any PP or a secure third party randomly generates a orthogonal $Q$ and conduct a column partition $Q=[Q^1,Q^2, \cdots, Q^M]$ where $Q^i\in \R^{d\times d_i}$ and send the i-th submatrix to the client i.
  \FOR {$t=1$ {\bfseries to} $T$}
  
  \STATE An action set $X_t = \{x_{t,i}\}^{K}_{i=1}$ arrives to the system while the client $j$ can only observe $x^{j}_{t,i} \in R^{d_j \times 1}$ and $x_{t,i} = [x_{t,i}^{1}, x_{t,i}^{2},\cdots, x_{t,i}^{M}]$
  
  \STATE AP collects $\tilde{x}_{t,i} = \sum_{j=1}^M Q^j x_{t,i}^{j}$ from client $1$ to $M$ (includes itself)
  
  \STATE AP samples $\tilde{\mu}_t\sim \mathcal{N}(\tilde{\theta}_t, v^2\tilde{\Lambda}^{-1}_t)$ and recommend the item $a_t = \arg \max_{i\in [K]} \tilde{x}_{t,i}^{\top}\tilde{\mu}_t$ 
  
  \STATE AP receives $r_t$ from the user
  
  \STATE AP updates its local estimator by 
  \begin{align*}
     \tilde{\Lambda}_{t+1} = \tilde{\Lambda}_t + \tilde{x}_{t,a_t}\tilde{x}_{t,a_t}^{\top}, \quad  \tilde{u}_{t+1} =  \tilde{u}_t + r_t\tilde{x}_{t,a_t}
  \end{align*}
  and 
  \begin{align*}
     \tilde{\theta}_{t+1} = \tilde{\Lambda}^{-1}_{t+1} \tilde{u}_{t+1}
  \end{align*}
  \ENDFOR 
\end{algorithmic}
\end{algorithm}

\subsection{VFTS Loseless Proof}
Following the established framework, we prove the following lemma to show VFTS is loseless.

\begin{theorem}
Given the fixed sequence $\{x_{t,a}\}_{t\in[T], a\in A}$ and corresponding return sequence $\{r_{t,a}\}_{t\in[T], a\in A}$, for any time $t\in [T]$, we have 
\begin{itemize}
    \item The estimated reward in the VFTS follows the same distribution as the one in the LinTS, i.e., $\tilde{x}_{t,a}^{\top}\tilde{\mu}_t$ is the same distribution as $x_{t,a}^{\top}\hat{\mu}_t$.
\end{itemize}
\end{theorem}
\begin{proof}
From the design of the Algorithm~\ref{alg:VFTS}, we have $$\tilde{\theta}_t\sim\mathcal{N}(Q\tilde{\mu}_t, v^2 Q\hat{\Lambda}_t^{-1}Q^{\top}).$$
Recall that in the LinTS, $x_{t,a}^{\top}\hat{\mu}_t\sim \mathcal{N}(x_{t,a}^{\top}\hat{\theta}_t, v^{2}x_{t,a}^{\top}\hat{\Lambda}_t^{-1}x_{t,a})$
Thus we know $\tilde{x}_{t,a}^{\top}\tilde{\mu}_t\sim \mathcal{N}(x_{t,a}^{\top}\hat{\theta}_t, v^2 x_{t,a}^{\top}\hat{\Lambda}_t^{-1}x_{t,a})$ follows the same distribution as $x_{t,a}^{\top}\hat{\mu}_t$ for all $a\in \mathcal{A}_t$.
\end{proof}

\section{Privacy Analysis}
\label{appendix:privacy}
Here we proof Theorem~\ref{theorem:masked-secure}.
\begin{proof}

Consider an arbitrary rotational matrix $R$ with its inverse $R^{-1}$.
Masked data $D$ is the result of multiplying the orthogonal mask $Q_1$ and the user context data $X_1$, i.e., $D = Q_1X_1$. 
Let $Q_2 = Q_1R$ and $X_2 = R^{-1}X_1$. Since an orthogonal matrix is still orthogonal after rotation, we have $D = Q_1X_1 = Q_2R^{-1}RX_2 = Q_2X_2$. 
As $R$ is arbitrary, we have infinite number of $Q_2$ and $X_2$.
\end{proof}

\section{Complexity Analysis}
\label{appendix:complexity}
\subsection{Computational Complexity}
We conduct computational analysis on both VFCB protocols. We assumed that for computational cost calculation, a single addition, multiplication, comparison or random generation operation for one element of the matrix has 1 operation cost. We then derive the standard addition, multiplication, dot for matrix operation base on this assumption to calculate the synthetic computational cost. These operation cost are unoptimized naive implementation cost to give an upper bound. For known Big-O cost such as n-dimensional matrix inverse ($O(n^3)$), we directly use them in our calculation.

As illustrated in Algorithm~\ref{alg:vfUCB} and Algorithm~\ref{alg:VFTS}, our VFCB protocols can be divided into the following stages:
\begin{enumerate}
    \item Mask initialization
    \item Selection process ($\forall t \in T$)
    \item Update process ($\forall t \in T$)
\end{enumerate}
Hence, we perform complexity calculation for each stages of our VFCB protocols and combine them to provide an overall computational complexity. Note that the Mask initialization process is shared between both VFUCB and VFTS.

\subsubsection{VFUCB}
\subsubsection*{Stage (1)}
In stage (1), our approach to generate orthogonal matrix $Q$ is to generate a random matrix $M \in R^{d \times d}$, which cost $o(d^2)$, and conduct Gram-Schmitz orthogonalization ($O(d^3)$) on it to ensure that the resulted mask $Q \in R^{d \times d}$ matrix is orthogonal. The overall computation cost for this stage is $O(d^3)$.

\subsubsection*{Stage (2)}
In stage (2), $\forall t \in T$, we need to consider VFUCB selection process. 
The masking and adding process for $\tilde{x}_{t,i}$ for every item $K$ is $O(K \times d^2) + O(K \times M \times d)$.
The process of calculating $\tilde{r}_{t,i}$ for every item $K$ is $O(K \times d)$.
The process of calculating $\Lambda^{-1}_{t}$ is $O(d^3)$, which is shared by all item $K$.
The process of calculating $\tilde{B}_{t,i}$ for every item $K$ is $O(K \times d^2)$.
The process of calculating for every item $K$ $\arg \max a_t$ is $O(K)$.
The overall cost of this stage is $O(T\times(K\times M\times d + K\times d^2 + d^3))$.

\subsubsection*{Stage (3)}
In stage (3), $\forall t \in T$,  we need to consider VFUCB update process.
The process of calculating $\Lambda_{t+1}$ is $O(d^2)$.
The process of calculating $\bar{u}_{t+1}$ is $O(d)$.
The process of calculating $\bar{\theta_{t+1}}$ is $O(d^2)$.
The overall cost of this stage is $O(T \times d^2).$

\subsubsection*{Overall Cost}
By combining the previous three stages, the overall cost for VFUCB is $O(T\times(K\times M\times d + K\times d^2 + d^3))$. Noted that the cost is dynamically determined by the size of each variables.

\subsubsection{LinUCB}
For LinUCB, stage (1) does not exist. Additionally, in stage (2), the cost for obtaining $\tilde{x}_{t,i}$ from O3M is not necessary. Hence the overall cost is the direct result of removing these two cost, which is $O(T\times(K\times d^2 + d^3)$.
\subsubsection{VFTS}
\subsubsection*{Stage (1)}
In stage (1), our approach on O3M is identical with VFUCB. Hence the overall computation cost for this stage is $O(d^3)$.
\subsubsection*{Stage (2)}
In stage (2), $\forall t \in T$, we need to consider VFTS selection process. 
The masking and adding process for $\tilde{x}_{t,i}$  for every item $K$ is $O(K \times d^2) + O(K \times M \times d)$.
The process of calculating covariance matrix $v^2\Lambda^{-1}_{t}$ is $O(d^2)$. Note that the $\Lambda^{-1}_{t}$ is compute in stage (3) update process.
The process of sampling of $\tilde{\mu}_{t}$ from multivariate normal distribution is $O(d^3)$ as Chelosky decomposition is $O(d^3)$.
The process of calculating $\tilde{x}_{t,i}^{\top} \tilde{\mu}_t$ for every item $K$is $O(K \times d)$.
The process of calculating $\arg \max a_t$ for every item $K$ is $O(K)$.
The overall cost of this stage is $O(T \times(K \times M \times d + K \times d^2 + d^3))$. 
\subsubsection*{Stage (3)}
In stage (3), $\forall t \in T$,  we need to consider VFTS update process.
The process of calculating $\Lambda_{t+1}$ is $O(d)$.
The process of calculating $\Lambda_{t+1}^{-1}$ is $O(d^3)$.
The process of calculating $u_{t+1}$ is $O(d)$.
The process of calculating $\theta_{t+1}$ is $O(d^2)$.
The overall cost of this stage is $O(T \times d^3)$.
\subsubsection*{Overall Cost}
By combining the previous three stages, the overall cost for VFTS is $O(T\times(K\times M\times d + K\times d^2 + d^3))$. Noted that the cost is also dynamically determined by the size of each variables.
\subsubsection{LinTS}
For LinTS, similarly, stage(1) does not exist. Furthermore, the cost of O3M process does not exist as well. Hence the overall stage (2) cost is reduced to $O(T \times(K \times d + d^3))$. Note that unlike LinUCB which also have $K\times d^2$ term during UCB value calculation, LinTS only have $K \times d$ term for calculation after sampling, therefore its overall cost reduce to $O(T \times(K \times d + d^3))$.

\subsection{Communication Cost}
We assume that communication cost for each element of the array is 1. For convenient reason, we give an upper bound of this process, which means that we also include the masked communication cost for the acting mask host PP in stage (1) and AP in stage (2). As shown in Algorithm~\ref{alg:vfUCB} and Alogrithm~\ref{alg:VFTS}, there are two stages of communication happened with VFCB.

\begin{enumerate}
    \item Mask Delivery (during initialization)
    \item Masked context delivery ($\forall t \in T$)
\end{enumerate}

In stage (1), the mask delivery require the Mask Generator PMG sends mask $Q^i$ to client $i$.  As $Q = [Q^1, Q^2, \cdots Q^j]$, the total mask delivery cost for $Q \in {d \times d}$ is $d^2$.

In stage (2), the masked sums are $Q^j x_{t,i}^{(j)} \in R^{d \times 1}$ for $i \in [K], j \in [M]$ Hence the masked sum delivery cost at each t is $\sum_i \sum_j d = K \times M \times d $, and the total mask delivery sum is $T \times K \times M \times d$.

Therefore, the overall communication cost is $d^2 + T \times K \times M \times d$, which can be simplify to $O(T \times K \times M \times d)$, given that in practice, $T \times K \times M > d$.

\section{Experiments}
\subsection{Experiments Environment}
\label{subsec:experiment-env}
Experiments on synthetic datasets and trail runs for real datasets are conducted on a workstation with AMD Epyc 7H12, 256G memory. The complete experiments for real-world dataset and hyperparameters search are conducted on cluster servers each with 23 core of Intel 8252C, 110G RAM. All experiment environments are built from docker image.

\subsection{Synthetic Data Experiment Setting}
\label{subsec:synthetic-setting}
In order to simulate the vertical federated settings, $\forall a\in[K]$, $x_{t,a}$ is partitioned into $M$ parts which are held by the $M$ th participant respectively. The number of context for each participant is derived from a partition vector $[1,\cdots,M]$, where $M$ represents the number of columns partitioned for participant $M$. Hence the context partition is $x_{t,a} = [x_{t,a}^{1},\cdots, x_{t,a}^{M}]$.

The central bandits experiments are performed on $d=100,K=10,T=5000$. The VFCB experiments are performed with the same parameters and a partition of $[20,20,20,20,20]$, which represents equal dimension number in 5 participants. Each experiment is repeated 5 times to provide a confidence interval. $\beta_t = 0.5$ is used for LinUCB and $v = 0.01$ is used for LinTS.

\subsection{Criteo Dataset Experiment Setting}
\label{subsec:criteo-setting}

Specifically, we follow the assumption that the user features are all numerical values and the item features are all categorical in \cite{DBLP:conf/mlsys/MalekzadehAHL20}. Following their method of feature engineering, we first use feature hashing to construct 3 hash-encoded values from  26 category values. After that, a pairing function is used to pair the 3 hash-encoded values to a single item label. Then, we filter the 40 most common labels as item label in experiment data. However, instead of just using the item label, we use the 3 encoded values as item features for that item label. This helps to capture the user-item interaction in our single parameter bandit setting. Finally, we build each log entry by concating the respective user features(dimension $d_u$), item features(dimension $d_i$), item label and the respective respond value $r$. We further carry out feature engineering by min-max scaling the feature columns and multiplied an individual scaling factor for each user features. The context features $x_{t,a} \in R^{d_u \times d_i}$ is obtained via outer addition between user features and item features for each log entry. 

Following the approach of \cite{DBLP:conf/www/LiCLS10}, we use the unbiased offline evaluation policy for our bandits. Similarly, we report the relative CTR as the CTR metric. The $CTR_{randoom}$ is obtained by averaging the random policy among 5 random states. After a grid search of hyperparameter, $\beta_t = 0.6$ is used for LinUCB and all subsequent partial experiments for the dataset. We report each partial LinUCB experiment with 5 different random partitions.

\end{document}